\documentclass{article}

\usepackage{amsmath,amsfonts,bm}
\usepackage{xcolor}

\def\eqref#1{equation~\ref{#1}}

\def\ceil#1{\lceil #1 \rceil}
\def\floor#1{\lfloor #1 \rfloor}
\def\1{\bm{1}}

\def\mA{{\bm{A}}}
\def\mB{{\bm{B}}}

\def\mF{{\bm{F}}}

\def\mI{{\bm{I}}}

\def\mL{{\bm{L}}}

\def\mP{{\bm{P}}}
\def\mQ{{\bm{Q}}}
\def\mR{{\bm{R}}}
\def\mS{{\bm{S}}}

\def\mV{{\bm{V}}}

\DeclareMathAlphabet{\mathsfit}{\encodingdefault}{\sfdefault}{m}{sl}
\SetMathAlphabet{\mathsfit}{bold}{\encodingdefault}{\sfdefault}{bx}{n}

\def\gC{{\mathcal{C}}}

\def\gL{{\mathcal{L}}}

\def\gN{{\mathcal{N}}}

\def\gU{{\mathcal{U}}}

\def\gW{{\mathcal{W}}}
\def\gX{{\mathcal{X}}}
\def\gY{{\mathcal{Y}}}

\def\sZ{{\mathbb{Z}}}

\newcommand{\E}{\mathbb{E}}

\newcommand{\norm}[1]{\left\Vert#1\right\Vert}

\newcommand{\abs}[1]{\left\vert#1\right\vert}

\newcommand{\set}[1]{\left\{#1\right\}}

\newcommand{\brac}[1]{\left [#1\right ]}

\newcommand{\Real}{\mathbb R}

\newcommand{\one}{\mathbf{1}}
\def\floor#1{\lfloor #1 \rfloor}

\usepackage[main, final]{neurips_2026}
\usepackage[utf8]{inputenc} 
\usepackage[T1]{fontenc}    
\usepackage{hyperref}       
\usepackage{url}            
\usepackage{booktabs}       
\usepackage{amsfonts}       
\usepackage{nicefrac}       
\usepackage{microtype}      
\usepackage{xcolor}         

\usepackage{graphicx}
\usepackage{subcaption}
\usepackage{multirow}
\usepackage{arydshln}
\usepackage{wrapfig}
\usepackage{adjustbox}

\usepackage{amsmath}
\usepackage{amssymb}
\usepackage{mathtools}
\usepackage{amsthm}
\usepackage{algorithm}
\usepackage{algorithmic}

\usepackage[capitalize,noabbrev]{cleveref}

\theoremstyle{plain}
\newtheorem{theorem}{Theorem}[section]

\newtheorem{lemma}[theorem]{Lemma}

\theoremstyle{definition}

\theoremstyle{remark}

\usepackage[textsize=tiny]{todonotes}
\title{Space Group Conditional Flow Matching}

\author{
  Omri Puny \\
  Weizmann Institute of Science
  \And
  Yaron Lipman \\
  FAIR at Meta
  \And
  Benjamin Kurt Miller \\
  FAIR at Meta
}

\newif\ifalexmp
\alexmptrue
\alexmpfalse
\newcommand{\alex}[1]{\ifalexmp {\textcolor{blue}{#1}} \fi}

\begin{document}

\maketitle

\begin{abstract}
Inorganic crystals are periodic, highly symmetric arrangements of atoms in three-dimensional space. Their structures are constrained by the symmetry operations of a crystallographic \emph{space group} and restricted to lie in specific affine subspaces known as \emph{Wyckoff positions}. The frequency an atom appears in the crystal and its rough positioning are determined by its Wyckoff positions. Most generative models that predict atomic coordinates overlook these symmetry constraints, leading to unrealistically high populations of proposed crystals exhibiting limited symmetry.
We introduce Space Group Conditional Flow Matching (SGFM), a novel generative framework that samples significantly closer to the target population of highly symmetric, stable crystals. We achieve this by conditioning the entire generation process on a given space group and set of Wyckoff positions; specifically, we define a conditionally symmetric noise  distribution and a group-conditioned, equivariant, parametric vector field that restricts the motion of atoms to their initial Wyckoff position. Our form of group-conditioned equivariance is achieved using an efficient reformulation of \emph{group averaging} tailored for symmetric crystals. It reduces the computational overhead of symmetrization to a negligible level.
SGFM achieves state-of-the-art performance in \textit{De Novo Generation} on the MP-$20$ benchmark.
\end{abstract}

\section{Introduction}

The properties of many technologies are determined by the quality of the solid-state crystals from which they are derived. Better materials lead to better technology. At the same time, the space of possible crystal structures is vast and underexplored. Discovery of novel crystalline structures is critical for material design and technological progress. We explore generative modeling to propose novel crystals. Crystals are solid materials characterized by a periodic arrangement of their constituent atoms. The crystalline structure can be represented by three components: lattice parameters (the geometry of the repeating unit cell), fractional coordinates (the position of each atom within the cell), and the atomic number of each atom.  A crystal's space group, a discrete subgroup of the Euclidean group, describes the symmetry of the crystal. The space group restricts atomic locations and lattice structure. Atomic constraints are called Wyckoff positions, sets of symmetrically equivalent points within a unit cell. The Wyckoff positions of a space group partition the unit cell according to the structure of the orbits induced by the group (see \cref{fig:schematic} for a $2$D example).

Most existing generative methods overlook key crystallographic properties: space group and Wyckoff positions. This limits their ability to generate non-trivially symmetric crystals\alex{, which we empirically demonstrate in \cref{sec:exp}}. The empirical distribution from an experimental and computational database of stable crystals (MP-20 \citep{jain2013materials}) contains only 1.7\% of so-called \emph{P1} crystals, i.e., the least symmetric crystals with only translational symmetry. This distribution is not merely a bias in the database; symmetry plays a key role in determining optical, electrical, magnetic, and structural properties of materials \cite{Chen2022Topological_semimetal,Tang2019topological_materials, malgrange2014symmetry, yang2005introduction}.

We develop a generative model that samples crystals conditioned on a given space group and associated Wyckoff positions. This approach offers two key benefits: (1) it provides greater control over the structure and symmetry of the generated crystals, and (2) it can leverage the lower-dimensional constraints imposed by Wyckoff positions for improving generation. In contrast to prior methods that incorporate space group information but rely on projection steps to correct atomic placements \cite{jiao2024space, levy2025symmcd}, our model is designed to inherently preserve the assignment of atoms to their designated Wyckoff positions throughout the generation process. This more natural inductive bias translates directly into superior empirical performance.

Our proposed model, \textit{\underline{S}pace \underline{G}roup Conditional \underline{F}low \underline{M}atching} (SGFM), is based on the \textit{Flow Matching} (FM) generative framework \cite{albergo2022building,liu2022flow,lipman2022flow}. We chose the FM framework for two key reasons: it allows us to use customized source distributions and provides known conditions for generating data that respects specified symmetries \cite{kohler2020}.
Based on these advantages, we designed SGFM with two main components: A space group and Wyckoff position conditioned noise prior, which has positive support only for crystal structures that adhere to the symmetry constraints described by the Wyckoff positions; and a group conditioned equivariant vector field, which is a neural network architecture that supports arbitrary space group equivariance. Equivariance is achieved through Group Averaging (GA)~\cite{yarotsky2022universal}, a symmetrization technique that projects arbitrary functions onto their equivariant versions. Although GA is typically computationally expensive, we introduce an efficient formulation tailored for symmetric crystals, reducing the computational overhead of the symmetrization to a negligible level. Our contributions are:

\begin{itemize}
    \item We define symmetric crystal generation in terms of distributional symmetry properties (\cref{ssec:problem}), and introduce the conditions that enable flow-based models to sample from such distributions (\cref{ssec:conditions}). We instantiate this flow model as SGFM (\cref{ssec:SGFM}), which consists of a noise prior conditional on Wyckoff position along with a space group-equivariant vector field, ensuring that the generated crystals preserve the specified symmetries.
    \item We propose a novel implementation of GA for symmetric crystals, equivalent in output to the standard GA but \emph{significantly} more efficient (see \cref{fig:main_fig} (c)).
    \item SGFM achieves state-of-the-art S.U.N. rate on the MP-$20$ dataset in \emph{de novo generation}.
    \alex{\item SGFM achieves state-of-the-art S.U.N. rate on the MP-$20$ dataset. On the large-scale dataset Alex-MP-$20$, SGFM produces a competitive S.U.N. rate. Additionally, 90\% of S.U.N. structures exhibit a non-trivial symmetry. \emph{MatterGen} fails to produce any non-trivially symmetric S.U.N. structures, even after relaxation.}
\end{itemize}

\begin{figure}[]
    \centering  
    \resizebox{\textwidth}{!}{
\includegraphics[width=0.9\linewidth]{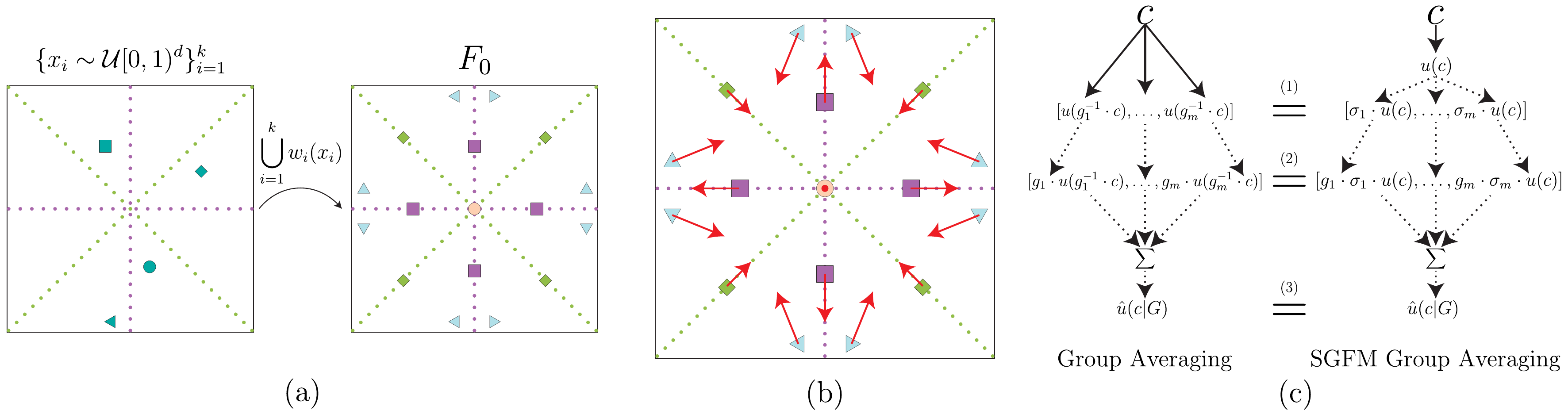}}
\vspace{-15pt}
    \caption{The main components of SGFM. (a) Wyckoff position noise prior. General points are sampled randomly and projected according to the conditioned Wyckoff positions. (b) Space group equivariant vector field. The equivariance of the model combined with the $G$-symmetry of the input crystal ensures that atoms preserve their symmetry structure.
    (c) Comparison between standard group averaging and our optimized formulation. The $G$-symmetry of the input crystal allows space group operations to be replaced by their corresponding permutations (denoted using $\sigma$). When combined with the $S_n$-equivariance of $u$, this reduces the number of required forward passes per crystal. For visualization purposes, heavy arrows indicate expensive model forward passes.}
    \label{fig:main_fig}
    \vspace{-10pt}
\end{figure}

\vspace{-10pt}
\section{Preliminaries}\label{sec:Preliminaries}

\textbf{Equivariance \& Invariance.}
A function $\varphi:\gX\to\gY$ is equivariant with respect to a group $G$ if the action of any group element on the input corresponds to a consistent transformation of the output. Equivariance implies $\varphi(g\cdot x)=g\cdot \varphi(x)$ for all $x\in \gX$ and $g\in G$. Invariance is a simplified case of equivariance, with all $g \in G$ mapping to a trivial group action on the output space $\varphi(g\cdot x)=\varphi(x)$. Invariance and equivariance also extend to group products: Let $(g_1,g_2)\in G_1\times G_2$, $\varphi$ is $G_1\times G_2$ equivariant if $\varphi((g_1,g_2)\cdot x)=(g_1,g_2)\cdot \varphi(x)$. Additionally,  \emph{$G$-invariant distributions} refer to distributions which have a $G$-invariant density function.

\textbf{Crystal Representation.}
 A crystal can be represented using the tuple $c' = (\mL, \mF, \mA)\in \gC'$, where $\mL \in \Real^{3 \times 3}$ is a positive-determinant, invertible matrix defining the lengths, angles, and orientation of the positive-volume unit cell; $\mF \in [0,1)^{n \times 3}$ denotes the fractional coordinates of $n$ atoms within the unit cell; and $\mA \in \set{0,1}^{n \times h}$ is a one-hot matrix indicating the atom types in $\mF$ from a set of $h$ atom types.
We adopt the space group–conditioned lattice parameterization proposed by \citet{jiao2024space}, which replaces 
$\mL$ with a rotation-invariant vector $k\in\Real^6$ of coefficients of a symmetric matrix basis, and represent a crystal as $c=(k,\mF,\mA) \in \gC$. Further details are provided in \cref{app:lattice}.

\textbf{Symmetries \& Crystals.} 
Our method focuses on how symmetry groups act on crystals.
The permutation group $S_n$ acts on $c$ by permuting the rows of $\mF$ and $\mA$. If $\sigma\in S_n$ is represented by a permutation matrix $\mP\in\set{0,1}^{n\times n}$ then $\sigma\cdot c = (k,\mP\mF,\mP\mA)$. 
The group of isometries of $\mathbb{R}^{3}$, i.e, the Euclidean group $E(3)$, acts on $c$ by applying an orthogonal transformation $\mR\in O(3)$ and a translation $\tau\in\Real^3$ to the coordinates. 
For group element $g=(\mR,\tau)\in E(3)$, the action  is defined as $g\cdot c=(k,\mF\mR^T+\one_n\tau^T-\floor{\mF\mR^T+\one_n\tau^T},\mA)$ where $\one_n\in\set{1}^n$ is a column vector of ones and $\floor{\cdot}$ is the element-wise floor function.
We specify the action of the product group on $c$. Let  $(g,\sigma)\in G\times S_n$ and $(g,\sigma)\cdot c \coloneqq g\cdot(\sigma\cdot c)$. If $G\leqslant E(3)$ then $g\cdot(\sigma\cdot c) = \sigma\cdot(g \cdot c)$; namely, the operators commute \citep{puny2021frame}.

\begin{wrapfigure}[20]{r}{0.4\textwidth}
  \centering
  \vspace{-18pt}
  \includegraphics[width=0.3\textwidth]{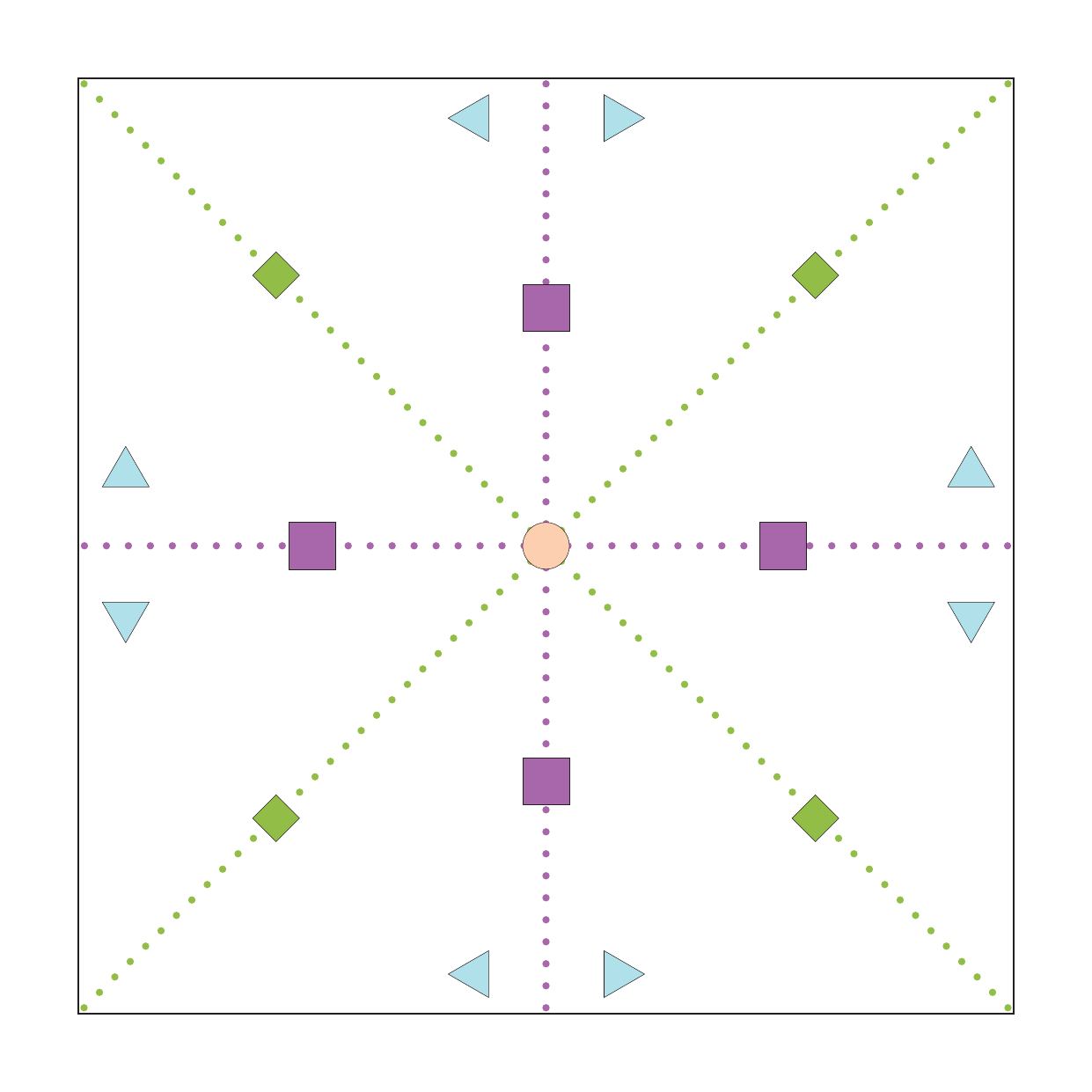}
  \vspace{-5pt}
  \caption{A $2$D example of a unit cell with p4mm symmetry. Applying a group element to this set permutes ``atoms'' of the same type (shape and color). The symmetry divides the unit cell (black box) into four Wyckoff positions: the center, horizontal and vertical coordinate axes, diagonal axes, and general position (white space).}
  \label{fig:schematic}
\end{wrapfigure}

\textbf{Space Groups.}
The space group concept formalizes the intrinsic symmetry of a crystal. 
If $c\in\gC$ is a crystal and $G\leqslant E(3)$ is its symmetry space group, then for any $g\in G$ there exists a permutation $\sigma\in S_n$ that satisfies the relation $g\cdot c=\sigma\cdot c$, a property we will denote as $G$-\textit{symmetry}. 
In essence, any action of the space group is equivalent to a permutation of the atom positions, \cref{fig:schematic} visualizes this property using a $2$D example.
The p4mm symmetry group includes rotations by angles of $\frac{\pi z}{2}$ for $z\in\mathbb{Z}$. The corresponding permutation is invisible, without fabricated labels, because it rearranges the positions of identical shapes. 
Two Crystals $c_1,c_2\in\gC$ are \textit{Mutually $G$-Symmetric} if every space group element $g$ corresponds to the same permutation $\sigma$ on both crystals.
Formally, $c_1$ and $c_2$ are mutually $G$-symmetric if $g\cdot c_1=\sigma\cdot c_1 \iff g \cdot c_2=\sigma\cdot c_2$. There exist $230$ distinct space groups in three-dimensional crystallography.
Owing to the intrinsic periodicity of crystal structures, all corresponding subgroups are finite subgroups of the Euclidean group $E(3)$. 
For non-orthogonal lattice structures, the space group acts on fractional coordinates as elements of the special affine group $SA(3)$, rather than $E(3)$, \cref{fig:hex_example} demonstrates this in $2$D. In simpler terms, we apply the group action after mapping every lattice to the cube using $\boldsymbol{L}^{-1}$.

\begin{wrapfigure}[14]{r}{0.4\textwidth}
    \centering
  \vspace{-12pt}
  \includegraphics[width=0.3\textwidth]{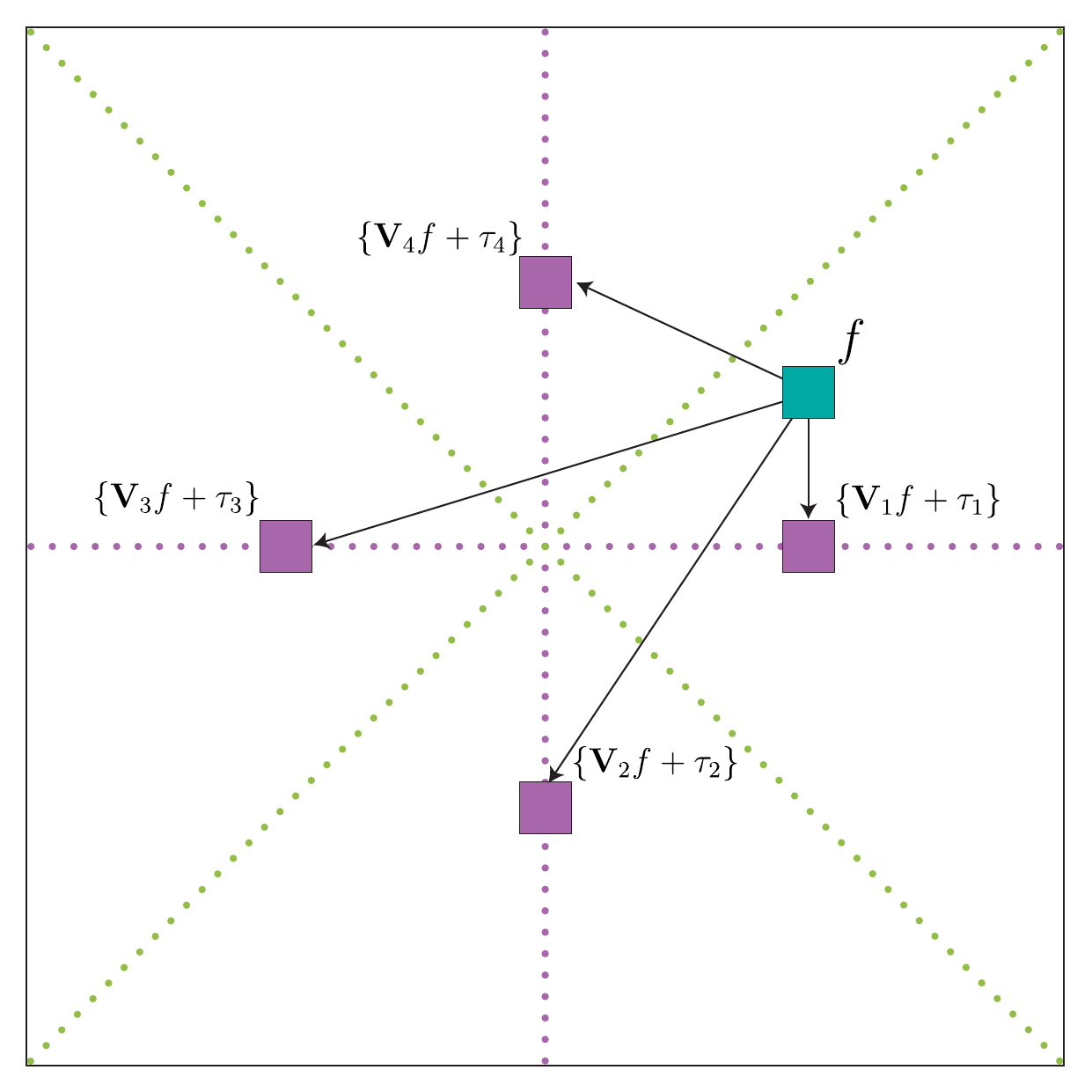}
  \vspace{-5pt}
  \caption{Projection of a coordinate $f$ using Wyckoff position $w$.}
  \label{fig:wyckoff}
\end{wrapfigure}

\textbf{Wyckoff Positions.}
Intuitively, Wyckoff positions of a space group $G$ indicate regions with specific symmetry properties. Atoms in general position occupy the least symmetric position in the crystal, appear most frequently in the unit cell, and enjoy the fewest restrictions on their coordinates. Meanwhile, atoms in one of the several special positions occupy a region of higher symmetry, appear less frequently in the unit cell, and are restricted to lie in low-dimensional affine subspaces.
More formally, a Wyckoff position $w$ is defined by a set of $\abs{w}=m$ affine projections $\{ (\mV_i,\tau_i)\}_{i=1}^m$ onto the corresponding affine subspace, with $\mV_i\in\Real^{3\times 3}$ and $\tau_i\in\Real^3$. We denote projection of $f\in[0,1)^3$ onto each of these $m$ affine subspaces by,
\begin{equation}
    w(f) \coloneqq \left\{ \mV_i f+\tau_i  - \lfloor \mV_i f+\tau_i \rfloor\right\}_{i=1}^{m}.
\end{equation}
If $w$ is a Wyckoff position of a space group $G$, for every $y\in w(f)$ we have the property 
$w(f)= G\cdot y$ where $G\cdot y \coloneqq \set{g\cdot y \mid g\in G}$ 
denotes the \emph{orbit} of $y$ under $G$. This is visualized in \cref{fig:wyckoff}; it illustrates how $f$ is mapped to an orbit induced by $G$ through $w$.
Each affine transformation $(\mV_i,\tau_i)$ identifies with a subgroup $G'\leqslant G$ that fixes points on its corresponding image. 
Namely, $y_i=\mV_i f+\tau_i  - \lfloor \mV_i f+\tau_i \rfloor$ for some $f\in[0,1)^3$ if and only if $G'=\set{g\in G \mid g\cdot y_i=y_i}$.
That means that $G'=G_{y_i}$, the \textit{site-symmetry} (stabilizer) group of $y_i$. We define the crystal $c$ (with fractional coordinates $\mF$) to be \textit{$\gW$-constructable} with respect to  $\gW=(w_1,\ldots,w_k)$ if there exist $k$ points $\set{x_i\in[0,1)^3}_{i=1}^k$ in the unit cell such that 
$\mF= \bigcup_{i=1}^k w_i(x_i)$
up to a permutation.  In other words, the coordinates in $\mF$ must be partitionable into complete symmetry orbits, one for each entry in $\gW$, with no additional points. This requires the symmetry relations within each orbit, rather than merely requiring individual coordinates to lie on the corresponding Wyckoff subspaces.

\textbf{Flow Matching (FM)} is a generative modeling framework that transforms samples from a simple base distribution $p_0$ into a complex target distribution $p_1$ using a time-dependent diffeomorphic
map, called a flow, $\psi:\brac{0,1}\times \gX\to\gX$. This flow is defined through the differential equation:
\begin{equation}\label{eq:FM_ODE}
    \frac{\mathrm{d}}{\mathrm{d}t}\psi_t(x)=v_t(\psi_t(x)),\quad
    \psi_0(x) = x
\end{equation}
\looseness=-1000
where $v_t:\brac{0,1}\times \gX\to\gX$ is a vector field that governs the evolution of the $\psi_t$.
The flow induces a time-dependent probability density path $p_t:\brac{0,1}\times\gX\to\Real$ starting at $p_0$ and ending with $p_1$. 
FM trains a parametric approximation $u_t$ of the true vector field $v_t$ \cite{lipman2024flowmatchingguidecode} by solving a regression objective. A conditional flow $\psi(\cdot\,|\,y):\brac{0,1}\times \gX\to\gX$ transports the entire base distribution to a single target point $y\in\gX$ and is governed by vector field $v_t(\cdot\,|\,y)$, which is easy to compute, unlike $v_t(\cdot)$. \citet{lipman2022flow} demonstrated that optimizing $u_t$ to match $v_t(\cdot\,|\, y)$ with regression leads to the same optimum as matching the marginal vector field $v_t$.

\section{Method}
This section presents our proposed model, SGFM, a flow matching-based generative approach designed to sample crystal structures conditioned on specified space groups and Wyckoff positions. We start by formalizing the problem and defining  the target distribution we aim to sample from in \cref{ssec:problem}. Then, \cref{ssec:conditions} outlines sufficient conditions for a flow model to sample from this distribution. \Cref{ssec:SGFM} provides a detailed overview of the model, including its key components, the noise prior and vector field. Finally, \cref{ssec:training} presents the training details of SGFM.

\subsection{Problem Definition}\label{ssec:problem} 
Given a finite set of crystals $\set{c_1, \ldots, c_m}$, each associated with a space group and Wyckoff positions $(G_i, \gW_i)$ and drawn from an unknown target distribution $q$, our goal is to design a generative model that samples crystals $c \sim p_1$ such that $p_1 \approx q$. To incorporate the structural information encoded by $G$ and $\gW$, we factorize the distribution as $p_1(c) = p_1(c \mid G, \gW)q(G, \gW)$, enabling us to model the crystal distribution conditionally on $G$ and $\gW$. While $G$ and $\gW$ are jointly sampled from the empirical distribution, the proposed generative model focuses on sampling crystals from the conditional distribution $c \sim p_1(\cdot \mid G, \gW)$, which satisfies the following properties:
\begin{itemize}
    \item $p(\cdot\,|\, G,\gW)$ is a \textit{$G$-symmetric} distribution, meaning that $p(c\,|\, G,\gW)>0$ only if $c$ is $G$-symmetric. The distribution assigns positive probability exclusively to crystals for which $G$ is their corresponding space group.
    \item $p(\cdot\,|\, G,\gW)$ is a \textit{$\gW$-constructable} distribution, meaning that $p(c\,|\, G,\gW)>0$ only if $c$ is $\gW$-constructable. 
    $p$ only supports crystals with fractional coordinates constructible by $\gW$.
\end{itemize}

The definitions provide a full description of the required distribution, but they can be further simplified. The following lemma shows that a $\gW$-constructable distribution implies a $G$-symmetric distribution, where $\gW$ denotes the Wyckoff positions associated with the space group $G$.

\begin{lemma}\label{lemma:W_to_G}
    Let $G\leqslant E(3)$ be a space group, $\gW$ a corresponding set of Wyckoff positions and $\mF\in[0,1)^{n\times 3}$ a $\gW$-constructable set of points in the unit cell, then $\mF$ is $G$-symmetric. 
\end{lemma}
The proof (\cref{app:W_to_G}) relies on the fact that the action of a group element on an orbit defines a bijection. Imposing these constraints on the conditional distribution substantially reduces its degrees of freedom, thereby simplifying generation.

\subsection{Theoretical Analysis}\label{ssec:conditions}
We now present the theoretical concepts underpinning the development of our generative model. \citet{kohler2020,rezende2019equivarianthamiltonianflows,song2023equivariant} establish the conditions under which a flow-based model can sample from a $G$-invariant distribution. For completeness, we restate this theorem and provide a condensed proof in  \cref{app:kohler_proof}.
Extending these results, we derive the conditions under which a flow-based model can sample from a distribution that is $\gW$-constructable.
Achieving this requires extending the standard framework with two modifications:
\begin{enumerate}
    \item Introducing a noise prior that is itself  $\gW$-constructable.
    \item Ensuring that the vector field model $u_t$ is equivariant w.r.t. the space group $G$ and the permutation group $S_n$. To do so, we extend the vector field (and corresponding flow) equivariance to the group product $G\times S_n$.
\end{enumerate}
\begin{theorem}\label{theorem:main_theorem}
    The probability path $p_t(x)$ defined by a flow generated by a $G\times S_n$ equivariant vector field $u_t$ from a $\gW$-constructable prior $p$ is $\gW$-constructable for all $t\in[0,1]$. 
\end{theorem}
This theorem follows directly from the lemma below (proof in  \cref{app:symmetric_flow}), which establishes that if the initial point $c_0$ of a $G\times S_n$ equivariant flow is a $G$-symmetric crystal structure, any point along the flow will be mutually $G$-symmetric with $c_0$. Based on the mutual $G$-symmetry, we demonstrate that the flow preserves the site-symmetry structure of $c_0$. This implies that if $c_0$ is $\gW$-constructable, then $\psi_t(c_0)$ is also $\gW$-constructable. \Cref{fig:main_fig} (b) visualizes the core idea behind the theorem, illustrating how the equivariant vector field constrains  atoms to move solely within the image of their Wyckoff position.
\begin{lemma}\label{lemma:symmetric_flow}
    Let $\psi_t$ be a $G \times S_n$ equivariant flow and $c\in \gC$ be $G$-symmetric and $\gW$-constructable, then $\psi_t(c)$ is $G$-symmetric and $\gW$-constructable. 
\end{lemma}

\subsection{SGFM}\label{ssec:SGFM}

In this section, we introduce the key components of SGFM, with  emphasis on the prior model and the learned vector field architecture. We explain how the previously outlined conditions are concretely implemented. The main focus lies in the interaction between the method and the crystal’s fractional coordinates, due to their strong dependence on the space group and Wyckoff positions.

\textbf{Noise Prior.} According to  \cref{theorem:main_theorem}, for the generated distribution to satisfy the conditions outlined in \cref{ssec:problem}, the noise prior must also satisfy the same constraints. \Cref{alg:sample_alg} presents a noise prior sampling pseudocode (with $2$D visualizations in \cref{fig:main_fig} (a)) that generates initial fractional coordinates compliant with these requirements. The sampling procedure for the lattice parameters and atom types is described in  \cref{ssec:training} for the reader’s convenience. The algorithm iterates over the set of Wyckoff positions and samples orbits induced by $G$ by projecting random points from the unit cell. It follows directly from the algorithm’s construction that $\mF_0$ is $\gW$-constructable.

\begin{algorithm}[tb]
   \caption{Sample $\mF_0\sim p_0(\cdot| G,\gW)$}
   \label{alg:smple}
\begin{algorithmic}[1]
   \STATE {\bfseries Input:} $\gW=\set{w_1,\ldots,w_k}$ s.t.~$w_i$ is a Wyckoff position of the space group $G$. 
  \STATE {\bfseries Output:}  $\mF_0\in[0,1)^{n\times 3}$ s.t.~$n=\sum_{i=1}^k\abs{w_i}$.
  \STATE set $\mF_0=[ \, ]$
  \FOR{$i = 1$ to $k$}
    \STATE Sample $x\sim \gU\brac{0,1}^3$
    \STATE $\mF_0 =$ Concatenate$(\brac{\mF_0,w_i(x)})$
  \ENDFOR
  \STATE return $\mF_0$  
\end{algorithmic}\label{alg:sample_alg}
\end{algorithm}

The sampling procedure also induces a uniform distribution on each periodic Wyckoff image with respect to its natural measure. Uniform sampling from $[0,1)^3$ corresponds to normalized Haar measure on the torus $\mathbb{T}^3=\mathbb{R}^3/\mathbb{Z}^3$. Since the linear parts $V_i$ of the Wyckoff maps used in our implementation have integer coefficients, $x\mapsto V_i x \pmod{1}$ defines a continuous group homomorphism onto its image and therefore pushes uniform Haar measure to uniform Haar measure on that image. The affine translation $\tau_i$ preserves uniformity. This holds for images of any dimension, with zero-dimensional positions yielding point masses.

\textbf{Space Group Conditional Vector Field.} As noted previously, $u_t$ must be $G\times S_n$ equivariant in order for the flow to be $G$-symmetric and $\gW$-constructable. The challenge lies in using a single $u_t$ model across crystals with varying space groups. Since some space groups (with non-orthogonal lattice structure) act on the fractional coordinates with special affine structured transformations, (see  \cref{fig:hex_example}),  using an $E(3)\times S_n$-equivariant model is non-trivial. In addition, modeling $G$-equivariance using a supergroup restricts the expressiveness of the model.  To address these limitations, we adopt Group Averaging (GA) \cite{yarotsky2022universal}, a symmetrization operator that projects a backbone model onto the space of $G$-equivariant functions. It is defined as: 
\begin{equation}\label{eq:GA}
    \hat{u}(c\, |\, G)= \frac{1}{|G|}\sum_{g\in G}g\cdot u(g^{-1}\cdot c).
\end{equation}
Applying GA to enforce space group symmetry addresses the previously mentioned limitations. Specifically, if the backbone $u$ is $S_n$ equivariant , then $\hat{u}(\cdot\,|\, G)$  is $G\times S_n$ equivariant \cite{puny2021frame}. Additionally, GA is not limited to subgroups of $E(3)$ and supports all finite groups. Importantly, \citet{puny2021frame} showed that symmetrization preserves the expressive power of the original model. However, a drawback of GA is its computational burden: directly applying \cref{eq:GA} increases the number of evaluations of $u$ by a factor of $\abs{G}$, which can be as large as $192$ for $3$D space groups (the average space group size in the MP-$20$ dataset is $\sim 45$). To mitigate this computationally intensive formulation, we leverage the fact that the inputs to $\hat{u}(\cdot\,|\, G)$ are $G$-symmetric crystals, allowing us to derive an efficient and equivalent formulation of GA specific to this case. 

\begin{lemma}\label{lemma:efficient_GA}
    Let $c\in C$ be a crystal, $G$ its space group, and $u$ an $S_n$ equivariant vector field. Then, \cref{eq:GA} can be equivalently rewritten as follows:
    \begin{equation}\label{eq:efficient_GA}
    \hat{u}(c\,|\, G)= \frac{1}{|G|}\sum_{g\in G}g\cdot \sigma_{g^{-1}| c}\cdot u(c)
\end{equation}
where $\sigma_{g^{-1}| c}\in S_n$ satisfies the equation $\sigma_{g^{-1}| c} \cdot c = g^{-1}\cdot c$.
\end{lemma}

The formulation presented in \cref{eq:efficient_GA} requires only a single evaluation of $u$, which dramatically improves the model efficiency. \Cref{fig:main_fig} (c) compares \cref{eq:GA} and \cref{eq:efficient_GA} and visualizes the efficiency gain. Furthermore, computing $\sigma_{g^{-1}| c}$ is computationally efficient, since we can decompose the problem according to the orbits of $c$, determined by $\gW$ (\cref{app:W_to_G}). At inference time, these permutations only need to be computed once for $c_0$, since  \cref{theorem:main_theorem} guarantees that the flow preserves $G$-symmetry structure. During training, permutations are computed only once during preprocessing for every data point. 
\Cref{tab:time_copmarison} compares the training and generation runtime between SGFM, a non-equivariant variant (no symmetrization), and  standard GA  highlighting the efficiency gains of our GA implementation compared to the standard GA, and demonstrating that its computational cost is comparable to using a backbone without symmetrization. Further details of this comparison are in \cref{app:runtime}.

\subsection{Training SGFM}\label{ssec:training} This section provides an overview of the SGFM training process. Let $c_1\in\gC$ be a crystal from the training set with a corresponding space group $G$ and Wyckoff positions $\gW$ . We will denote $c_0\sim p_0(\cdot| G,\gW,c_1)$ a sample from the conditional noise prior, $c_t=\psi_t(c_0\,|\,c_1)$ the conditional flow  where $c_t=(k_t,\mF_t,\mA_t)$, $v_t(c_t|c_1) = (v^k_t(c_t|c_1), v^{\mF}_t(c_t|c_1), v^{\mA}_t(c_t|c_1))$ is the conditional vector field and $\hat{u}_t(c_t| G)=(\hat{u}^k_t(c_t| G),\hat{u}^\mF_t(c_t| G),\hat{u}^\mA_t(c_t| G))$ is the prediction of the $G\times S_n$ equivariant vector field parametric model.

\textbf{Lattice Parameters.} As noted in \cref{sec:Preliminaries}, we represent lattice parameters using the group-conditioned representation from \citet{jiao2024space}, where $k \in \Real^6$ encodes the basis coefficients of a $3$D symmetric matrix constrained to $G$-specific subspaces. To sample $k_0\in \Real^6$, we first draw coefficients $k'\sim \gN(0,\mI)$ and apply a group condition mask: $k_0=k'\odot m(G)$, where $m(G)\in\set{0,1}^6$ is a group-dependent binary mask that zeros out the irrelevant coefficients. For hexagonal lattices we set the first entry of $k_0$ to be $\frac{-\log(3)}{4}$, as described in \cref{tab:l_cons}. $k_t$ is computed as a linear interpolation of $k_0$ and $k_1$, $k_t=(1-t)k_0+tk_1$ and the corresponding component of the conditional vector field is $v^k_t(c_t|c_1)=k_1-k_0$. Since the group action does not directly act on the lattice parameterization we need to apply the mask on $\hat{u}^k_t(c_t| G)$, both in training and after each generation step. The lattice optimization objective, $\gL^k(\theta)$, is defined as:
\begin{equation}
    \E_{t,q(c_1),p_0(c_0| G)}\norm{\hat{u}^k_t(c_t| G)\odot m(G)-(k_1-k_0)}^2_2
\end{equation}

\textbf{Atom Types.} When we predict atom types, we follow the modeling approach introduced in \citet{miller2024flowmm}, where atom types are represented using a $\set{-1,1}$ binary format instead of standard one-hot encoding. Specifically, $A_1 \in \set{0,1}^{n \times h}$ is converted into its binary representation $\Tilde{A}_1 \in \set{-1,1}^{n \times \ceil{\log_2 h}}$. To ensure $G$-symmetry in the initial sample $c_0$, atom types must be consistent within each orbit. Accordingly, we sample Gaussian noise $\gN(0,1)^{\ceil{\log_2 h}}$ per orbit and broadcast it to all atoms within that orbit to sample $\Tilde{A}_0\in \Real^{n \times \ceil{\log_2 h}}$. We define $\Tilde{A_t}=(1-t)\Tilde{A_0}+ t\Tilde{A_1}$ and $v^A_t(c_t|c_1)=\Tilde{A_1} - \Tilde{A_0}$. The atom types optimization objective, $\gL^A(\theta)$, is defined as: 
\begin{equation}
    \E_{t,q(c_1),p_0(c_0| G)}\norm{\hat{u}^A_t(c_t| G)- (\Tilde{A_1} - \Tilde{A_0})}^2_2
\end{equation}
Our GA formulation (\cref{eq:efficient_GA}) ensures that $\hat{u}^A_t(c_t \mid G)$ is $G$-invariant, meaning the atom type vector field is consistent across orbits, as required. A detailed explanation of this property can be found in \cref{app:atom_type}. During inference, we apply the sign function to convert the continuous atom type predictions into their binary representation.

\textbf{Fractional Coordinates.} \Cref{alg:sample_alg} describes a general procedure for sampling fractional coordinates that are both $G$-symmetric and $\gW$-constructable. To ensure $G$-symmetry of the conditional flow, the initial coordinates $\mF_0 \sim p_0(\cdot\mid G,\gW)$ must be $G$-symmetric with $\mF_1$. This requires that the order of elements and operators in $\gW$ match that used to generate $\mF_1$, which we precompute during preprocessing using the PyXtal library \cite{pyxtal}.
 We adopt the flat torus geometry of the unit cell, following the approach proposed by \citet{miller2024flowmm}, and define the conditional flow over the fractional coordinates,
\begin{align*}
    \psi_t(\mF_0|\mF_1)&=\mF_0+t\cdot \log_{\mF_0}(\mF_1), \\
    \log_{\mF_0}(\mF_1)&= \frac{1}{2\pi} \operatorname{atan2}(\brac{\sin(\mF_1 - \mF_0), \cos(\mF_1-\mF_0)}).
\end{align*}
Where $\log_{(\cdot)}(\cdot)$ is the element-wise logarithmic map 
over the flat tori.  In \cref{app:fractional_coordinates_flow} we demonstrate: (1) the conditional vector field $\log_{\mF_0}(\mF_1)$ is $G$-equivariant but with respect to a different representation of $G$. Let $g\in G$ then $\log_{g\cdot \mF_0}(g\cdot \mF_1)=g^{\star}\log_{\mF_0}(\mF_1)$ where $g^\star$ is defined by the homomorphism $(\mR,\tau)\mapsto \mR$; (2) $\psi_t(\mF_0|\mF_1)$ is mutually $G$-symmetric with $\mF_0$ and $\mF_1$, hence $G$-symmetric and $\gW$-constructable. The fractional coordinates optimization objective, $\gL^F(\theta)$,  is defined as:
\begin{equation}
    \E_{t,q(c_1),p_0(c_0|G,\gW)}\norm{\hat{u}^\mF_t(c_t| G) - \log_{\mF_0}(\mF_1)}^2_2
\end{equation}

Combining all components and hyperparameters $\lambda_k,\lambda_{\mF},\lambda_{\mA}\in\Real^+$ we obtain our training objective:
\begin{equation}
    \gL^{\text{SGFM}}(\theta)=\lambda_k\gL^k(\theta)+\lambda_{\mF}\gL^{\mF}(\theta)+\lambda_{\mA}\gL^{\mA}(\theta).
\end{equation}
\begin{table*}[t]
\centering
\small
\renewcommand{\tabcolsep}{1.6pt}
\caption{DNG evaluation. Models were trained on the MP-$20$ dataset. NFE refers to the number of generation steps. We evaluated all the S.U.N. metrics, except those marked with $^*$. \emph{DiffCSP++ (empirical)}, \emph{SGEquiDiff} results were computed with only 1,000 samples; others used 10,000 samples. Methods above the dividing line do not explicitly use Wyckoff positions; methods below use them.}
\resizebox{\textwidth}{!}{%
\begin{tabular}{lccclccclccclccc}
\toprule
\multirow{2}{*}{Model} & \multicolumn{1}{l}{\multirow{2}{*}{NFE}} & \multicolumn{2}{c}{Validity $(\%)\uparrow$} & $\,$                 & \multicolumn{3}{c}{Property $\downarrow$}        & $\,$                 & Stable $(\%)\uparrow$        & S.U.N. $(\%)\uparrow$ & S.S.U.N. $(\%)\uparrow$  &  $\,$                 & Stable $(\%)\uparrow$       & S.U.N. $(\%)\uparrow$ & S.S.U.N. $(\%)\uparrow$     \\ \cline{3-4} \cline{6-8} \cline{10-12} \cline{14-16} 
                       & \multicolumn{1}{l}{}                     & Structural           & Composition          &                      & $d_{\rho}$ & $d_{\text{elem}}$ & $d_{\text{cn}}$ &                      & \multicolumn{3}{c}{$E_{\text{hull}}<100\text{ meV/Atom}$} &                      & \multicolumn{3}{c}{$E_{\text{hull}}<0\text{ meV/Atom}$} \\
\midrule
CDVAE                  & $5000$                                   & $\textbf{100.00}$             & $86.70$              &                      & $0.688$    & $0.278$           & -               &                      & -                            & -  &-                        &                      & -                           & -   &-                      \\
ADiT                   & $500$                                    & $99.74$              & $\textbf{92.14}$              &                      & -          & -                 & -               &                      & $\textbf{72.0}$                       & $27.4$ & $18.21$                 &                      & $13.0$                      & $4.6$     & $3.57$                \\
FlowMM                 & $500$                                    & $96.86$              & $83.24$              &                      & $0.075$    & $0.079$           & $0.443$           &                      & $31.2$                       & $19.7$ & $14.38$                      &                      & $4.6$                       & $2.3$ &$1.85$                     \\
FlowLLM                & $250$                                    & $99.81$              & $89.05$              &                      & $0.660$     & $0.090$            & -               &                      & $67.9 
$                       & $21.9$  & $20.18$                  &                      & $14.2$                     & $3.6$   &$3.27$                  \\
MatterGen & $1000$ & \textbf{$100.00$} & $82.60$ && $0.206$ &$0.242$ &-& &-&$24.3$&-& &-& -&-\\
OMatG                  & $680$                                    & $95.05$              & $82.84$              &                      & $\textbf{0.060}$    & $\textbf{0.017}$           & $0.165$         &                      & $44.4$                       & $23.7$    &$13.28$                 &                      & $6.6$                       & $2.2$ & $1.24$                    \\
\hdashline
SymmCD                 & $1000$                                   & $90.34$              & $85.81$              &                      & $0.230$    & $0.400$           & -               &                      & -                            & -       &-                   &                      & -                           & -    &-                     \\
DiffCSP++ (empirical)  & $1000$                                   & $99.94$              & $85.12$              &                      & $0.235$    & $0.374$           & -               &                      & $31.4$                       & $21.1$ & $20.0$                    &                      & $7.2$                       & $4.0$  & $4.0$                   \\
DiffCSP++ (WyFormer)   & $1000$                                   & $99.66$              & $80.34$              &                      & $0.670$    & $0.098$           & -               &                      & -                            & -     & -                     &                      & -                           & $3.8^*$ & -                  \\
SGEquiDiff             & $1000$                                         & $99.25$              & $86.16$              & \multicolumn{1}{c}{} & $0.193$    & $0.209$           & -               & \multicolumn{1}{c}{} & -                            & $24.4^*$   &-               & \multicolumn{1}{c}{} & -                           & -      &-                   \\
\underline{SGFM (empirical)}       & $500$                                    & $99.87$              & $86.81$              &                      & $0.075$    & $0.181$           & $\textbf{0.076}$         &                      & $64.1$                       & $\textbf{30.3}$ & $\textbf{29.03}$                    &                      & $\textbf{14.6}$                      & $\textbf{6.9}$ & $\textbf{6.7}$                   \\
\underline{SGFM (WyFormer)}        & $500$                                    & $99.87$              & $84.76$              & \multicolumn{1}{c}{} & $0.237$    & $0.233$           & -               & \multicolumn{1}{c}{} & $48.4$                            & $22.6$    &$21.72$                      & \multicolumn{1}{c}{} & $10.6$                           & $4.7$ & $4.53$                       \\ 
\bottomrule
\end{tabular}}
\label{tab:dng_results}
\vspace{-10pt}
\end{table*}

\vspace{-20pt}
\section{Experiments}\label{sec:exp}
Our experiments are divided into two distinct tasks.
First, we evaluate SGFM on \textit{De Novo Generation} (DNG). The objective of this task is to simulate a materials discovery campaign by generating atom types, fractional coordinates, and lattice parameters from scratch. 
Second, to isolate the specific impact of our model on the relevant crystal structure components, we evaluate it on \textit{Crystal Structure Prediction} (CSP). By providing the model with conditioned atom types and requiring it to generate only the lattice parameters and fractional coordinates, we can more precisely measure how effectively SGFM utilizes Wyckoff positions and space group information, with performance determined by how closely the output matches a specific target crystal with shared atom types, space group and Wyckoff positions. While assuming known atom types and Wyckoff positions is impractical in applications, it serves as a rigorous ``unit test'' for the generative process.

\textbf{Metrics.}
In DNG, performance is measured by the model's ability to produce structures that are simultaneously stable, unique, and novel (S.U.N.). Stable inorganic crystals typically have non-\emph{P1} space groups, according to experimental databases, due to symmetries inherent in the electron density in crystalline matter.
We therefore additionally report the symmetric, stable, unique, and novel rate, i.e., S.S.U.N. \citep{kazeev2025wyckofftransformer} which is the S.U.N. rate with an additional restriction to crystals that \emph{do not} have \emph{P1} symmetry. \emph{P1} is the trivial symmetry in periodic systems containing only lattice translations. Given that stability evaluation via DFT (details in \cref{app:dft}) is computationally expensive, we include standard proxy metrics, such as \textit{structural \& compositional validity} and various property divergences, to further validate the quality of the generated crystals. For CSP, we utilize \texttt{StructureMatcher} \citep{ong2013python} to identify Match Rate (MR) and the root mean square error (RMSE). We use standard CSP settings. Further details are provided in \Cref{app:evaluation}.

\textbf{Datasets.} 
We evaluate DNG on \textit{MP}-$20$ \cite{jain2013materials}, with 45,231 diverse crystals from the Materials Project. 
We perform CSP on an additional four datasets: \textit{MPTS-}$52$, a time-ordered variant with $40,476$ crystals featuring larger unit cells; \textit{Perov}-$5$ \cite{castelli2012new}, with $18,928$ perovskites sharing a common structure but varying atom types, \textit{Carbon}-$24$ \cite{pickard2020cabon}, containing $10,153$ carbon crystals with diverse structures, and \textit{Alex-MP}-$20$, a large-scale set of $607,684$ crystals combining MP-$20$ and Alexandria data \cite{schmidt2022large, schmidt2022alexandria_2}.

\textbf{Baselines.} We compared SGFM to several state-of-the-art baselines. Methods that do not incorporate space group information in their generation process include \textit{CDVAE} \cite{xie2021crystal}, \textit{ADiT} \cite{joshi2025adit}, \textit{FlowMM} \cite{miller2024flowmm}, \textit{FlowLLM} \cite{sriram2024flowllm} , \textit{MatterGen} \cite{zeni2023mattergen}, DiffCSP \cite{jiao2023crystal}, \cite{sriram2024flowllm}, and \textit{OMatG} \cite{hoellmer2025omat}. In contrast, \textit{SymmCD} \cite{levy2025symmcd}, \textit{DiffCSP++} \cite{jiao2024space},  \textit{WyFormer} \cite{kazeev2025wyckofftransformer}, and \textit{SGEquiDiff} \cite{chang2025space} explicitly incorporate space group and Wyckoff position information. Additional details on each baseline are provided in \cref{app:baseline}.

\textbf{Model Details.} To model $\hat{u}_t$, we adopt the  architecture used in \citet{jiao2023crystal}, which utilizes a modified version of \textit{EGNN} \citep{satorras2022egnn} to handle fractional coordinates. The model applies sinusoidal embeddings to the fractional coordinates, ensuring invariance to lattice translations in addition to the space group equivariance. A description of the architecture and the hyperparameters used in each experiment are provided in  \cref{app:model_details}. For improved sampling quality, we apply inference anti-annealing \cite{yim2023fast} that adjusts the prediction velocity during generation.

\vspace{-5pt}
\subsection*{De Novo Generation}
We evaluated SGFM on \alex{two datasets:} MP-$20$\alex{ and the large scale dataset Alex-MP-$20$}. We employed two configurations of incorporating Wyckoff positions, one using Wyckoff positions taken from the train set, denoted \emph{empirical}, and one from the output of \emph{WyFormer}. 
All models were evaluated with $10,000$ samples, except \emph{DiffCSP++ (empirical)} and \emph{SGEquiDiff} that used only $1000$ samples. \emph{DiffCSP++ (WyFormer)} was reported by the authors using slightly different density functional theory settings. Wasserstein divergence for properties is computed between the test set and a structurally and compositionally valid subset of $1000$ generated samples. Training and sampling details are in \cref{app:train_details}.
  
Results for the MP-$20$ evaluation are summarized in \Cref{tab:dng_results}. SGFM, utilizing empirical Wyckoff positions, achieves state-of-the-art S.U.N. and S.S.U.N. rates across all energy thresholds. The comparison with SGFM (WyFormer) quantifies the effect of replacing empirical Wyckoff sampling with learned generation. Although empirical sampling yields stronger results, SGFM (WyFormer) still outperforms all compared baselines in S.U.N.\ and S.S.U.N.\ at the strict $E_{\mathrm{hull}}<0$ meV/atom threshold. This demonstrates that SGFM remains effective without access to the empirical Wyckoff distribution, while the performance gap between the two settings suggests room for improvement in learned Wyckoff generation. \alex{Our findings suggest that models lacking explicit space group and Wyckoff information struggle to generate stable, unique, and valid crystals with non-trivial symmetries. This trend is corroborated by the Alex-MP-$20$ evaluation (\Cref{tab:dng_alex}); while MatterGen yields a higher general S.U.N. rate than SGFM, its performance degrades significantly when symmetry constraints are introduced to the evaluation.}
To further compare SGFM's generation efficiency with other models, we measured generation times (across multiple batch sizes) on a single NVIDIA L40S GPU. The results are reported in \cref{tab:dng_time}. Despite requiring more parameters ($16.2$M compared to $12.2$M for DiffCSP++ and $5.5$M for SGEquiDiff) and operating on the full conventional cell rather than the asymmetric unit (which reduces memory-footprint), SGFM achieves comparable efficiency to SGEquiDiff and outperforms DiffCSP++.

\begin{table*}[htb]
    \centering
    \small
    \renewcommand{\tabcolsep}{1.6pt}
    \caption{Crystal Structure Prediction. MR denotes match rate. Best results in bold. Models below the dashed line receive ground-truth space-group and Wyckoff information; comparisons across the line therefore involve different conditioning information. 
    }
    \resizebox{\textwidth}{!}{%
    \begin{tabular}{lccccccccccc}
    \toprule
    \multirow{2}{*}{Model}&\multirow{2}{*}{\shortstack{Wyckoff\\Positions}}&\multicolumn{2}{c}{MP-$20$} & \multicolumn{2}{c}{MPTS-$52$}& \multicolumn{2}{c}{Perov-$5$} & \multicolumn{2}{c}{Carbon-$24$}& \multicolumn{2}{c}{Alex-MP-$20$}\\
   && \multicolumn{1}{c}{MR (\%) $\uparrow$} & \multicolumn{1}{c}{RMSE $\downarrow$} & \multicolumn{1}{c}{MR (\%) $\uparrow$} & \multicolumn{1}{c}{RMSE $\downarrow$} & \multicolumn{1}{c}{MR (\%) $\uparrow$} & \multicolumn{1}{c}{RMSE $\downarrow$} & \multicolumn{1}{c}{MR (\%) $\uparrow$} & \multicolumn{1}{c}{RMSE $\downarrow$} & \multicolumn{1}{c}{MR (\%) $\uparrow$} & \multicolumn{1}{c}{RMSE $\downarrow$} \\ \midrule
   CDVAE&\multirow{4}{*}{None}&$33.90$&$.1045$&$5.34$&$.2106$&$45.31$&$.1138$&$17.09$&$.2969$&-&-\\
   DiffCSP&&51.49&.0631&12.19&.1786&52.02&.0760&17.54&.2759&-&-\\
   FlowMM&&$61.39$&$.0566$&$17.54$&$.1726$&$53.15$&$.0992$&$\textbf{23.47}$&$.4122$&-&-\\ 
   OMatG&& \textbf{69.83}&$.0741$&$\textbf{27.38}$&$.1970$&$\textbf{83.06}$&$.3753$&-&-&$72.50$&$.1260$\\
   \hdashline
   DiffCSP++&\multirow{2}{*}{\shortstack{Ground\\truth}}&$80.27$&$.0295$&$46.29$&$.0896$&$98.44$&$.0430$&-&-&$83.18$&$.0188$\\
   \underline{SGFM}&&$\textbf{82.74}$&$.0288$&$\textbf{51.79}$&$.0827$&$\textbf{98.57}$&$.0188$&$\textbf{55.02}$&$.0952$&$\textbf{84.40}$&$.0198$\\
   \bottomrule
    \end{tabular}}
    \label{tab:csp_results}
    \vspace{-10pt}
\end{table*}
\subsection*{Crystal Structure Prediction}
The generative task in CSP requires sampling from the conditional target distribution $c\sim q(\cdot|\mA)$, where $\mA$ denotes a predefined atom type composition. This conditioning implies that during both training and generation $\mA_t=\mA$ for all $t\in\brac{0,1}$, effectively ignoring the loss term $\gL^{\mA}$ and the atom type component $\hat{u}^A_t(c_t| G)$.
For evaluation, a crystal structure is generated for each entry in the test set and then compared against the corresponding ground truth structure.

For the results in \cref{tab:csp_results}, SGFM and DiffCSP++ receive ground-truth space-group and Wyckoff information. This provides substantial additional information: for Wyckoff positions with zero degrees of freedom, the fractional coordinates are fully determined. We therefore treat this setting as a controlled evaluation of symmetry-conditioned generation, with SGFM and DiffCSP++ forming the appropriate matched comparison. SGFM achieves higher match rates on all four datasets for which DiffCSP++ results are available, while \cref{fig:SGFM_vs_diffcsp++} illustrates its sampling-efficiency gains. Results for methods without explicit symmetry conditioning provide useful context, but comparisons must account for the additional information supplied to SGFM and DiffCSP++.

We performed an ablation study to isolate the individual contributions of the SGFM components. \Cref{tab:model_type_ablation} presents a CSP evaluation on the MP-$20$ dataset across two key dimensions: the prior distribution (Wyckoff-conditioned vs. uniform over the unit cell) and the vector-field architecture (equivariant vs. non-equivariant). The results indicate that while a Wyckoff-conditioned prior provides a superior initialization, its synergy with the space-group equivariant vector field is the primary driver of improved performance, a finding consistent with our theoretical framework.

To evaluate our method in settings where Wyckoff positions are unknown, we utilize \textit{CSPML} \citep{kusaba2022crystal}, a metric learning-based model. Given an atomic composition, CSPML retrieves a similar composition from a template set along with associated space group and Wyckoff positions. We use these as inputs for the benchmark reported in \cref{tab:cspml}. The results indicate heuristic approaches for sampling Wyckoff positions provide useful information, but  less than the ground truth positions.

\section{Related Work}
\label{sec:related_work}
There is a growing body of literature about generative models for inorganic crystals. We focus here on works with similar inductive biases, namely explicit utilization of Wyckoff positions. 
We first consider works that generate atomic coordinates. \citet{cao2024space} created an autoregressive model that generates crystals sequentially in Wyckoff position's lexicographic order. \citet{jiao2024space, levy2025symmcd} produced diffusion models that both represent crystals within the asymmetric unit, a memory-efficient formulation that contains just one representative per orbit. Neither of these methods utilizes space group equivariance and both require projection steps to keep atomic coordinates within the target Wyckoff positions. A concurrently developed diffusion model by \citet{chang2025space} utilizes the asymmetric unit and, like our method, implements space group equivariance via group averaging without needing projection. However, working in the asymmetric unit does not allow for our efficient reformulation in \cref{eq:efficient_GA} since representing the crystal structure through the asymmetric unit does not yield a $G$-symmetric representation. Furthermore, keeping track of which atoms are in the asymmetric unit introduces some complexity in the implementation. As written, \citet{levy2025symmcd} do not address the crystal structure prediction problem. 
There is also a class of models that ignore explicit atomic coordinates  and generate coarse-grained Wyckoff positions alone, inspired by regression methods \cite{goodall2020predicting, goodall2022rapid}. \citet{zhu2024wycryst, kazeev2025wyckofftransformer, kelvinius2025wyckoffdiff} all take this approach. These models synergize with ours and can generate Wyckoff positions for SGFM to use in DNG in \cref{sec:exp}. Innovations in these methods, if conditional on atom types, could replace CSPML for CSP.
Further discussion of other relevant work is left for \cref{app:related_work_appendix}.

\section{Conclusions}
In this work, we introduced SGFM, an FM-based generative model for crystal structures, conditioned on space group and Wyckoff positions. By design, SGFM produces crystals that satisfy symmetry constraints, relying on sufficient conditions we formulated over the noise prior and vector field. We also implemented an efficient group averaging method, enabling the incorporation of space group equivariance into the vector field model with minimal overhead. Evaluated on the DNG task, SGFM achieved state-of-the-art performance. \alex{Our investigation also reiterated the challenge of producing S.S.U.N. structures, especially using models that do not explicitly model symmetry. Even with our more effective method, generating S.S.U.N. structures against large databases like Alex-MP-20 is extremely difficult with merely a 1.5\% generation rate.} Future directions include extending the model to an unconditional generation setting, where space group and Wyckoff positions are also generated rather than specified.  Another promising direction is to integrate property predictors with SGFM to guide generation toward desired physical properties while preserving crystallographic symmetry.

\begin{ack}
Research was partially supported by the Israeli Council for Higher Education (CHE) via the Weizmann Data Science Research Center and by research grants from the Estates of Tully and Michele Plesser and the Anita James Rosen and Harry Schutzman Foundations.
\end{ack}

\bibliography{icml_2026}
\bibliographystyle{icml2026}

\clearpage\appendix

\section{Proofs}
\subsection{\texorpdfstring{\cref{theorem:kohler}}{}}
\label{app:kohler_proof}
\begin{theorem}\label{theorem:kohler}
    The probability path $p_t(x)$ defined by a flow generated by a $G$-equivariant vector field $u_t$ from a $G$-invariant prior $p$ is $G$-invariant for all $t\in[0,1]$. 
\end{theorem}

\begin{proof}
The proof has two main parts. First, we will show that the flow $\psi_t$ defined by the $G$-equivariant vector field $u_t$ is $G$-equivariant. Then, we will use this property to demonstrate that the resulting probability path $p_t$ is $G$-invariant. As a reminder, the flow $\psi:\brac{0,1}\times\gX\to\gX$ is defined by the following ODE:  
\begin{align}\label{eq:flow_ode}
    \frac{d}{dt}\psi_t(x)&=u_t(\psi_t(x)) \\
    \psi_0(x)&=x
\end{align}
To demonstrate that $\psi_t$ is equivariant, we will show that two functions, $\varphi_t(x)\coloneq \psi_t(g\cdot x)$ and $\phi_t(x)=g\cdot\psi_t(x)$ (for arbitrary $g\in G$) satisfy the same ODE with identical initial conditions.
\begin{align*}
    \frac{d}{dt}\varphi_t(x)= \frac{d}{dt}\psi_t(g\cdot x)&=u_t(\psi_t(g\cdot x))=u_t(\varphi_t(x)) \\
    \varphi_0(x)=\psi_0(g&\cdot x)= g\cdot x
\end{align*}
\begin{align*}
\frac{d}{dt}\phi_t(x)= \frac{d}{dt}g\cdot\psi_t(x)=g\cdot\frac{d}{dt}\psi_t(x)&=g\cdot u_t(\psi_t(x))=u_t(g\cdot \psi_t(x))=u_t(\phi_t(x)) \\
\phi_0(x)=g\cdot&\psi_0(x)=g\cdot x
\end{align*}
The fourth equality uses the $G$-equivariance of $u_t$.
We can therefore conclude $\psi_t(g\cdot x)=g\cdot \psi_t(x)$ for every $x\in\gX, g\in G$ and $t\in\brac{0,1}$, which proves that $\psi_t$ is $G$-equivariant.

It remains to show that $p_t$ defines an invariant probability path.
\begin{align*}
    p_t(g\cdot x)&=p_0(\psi^{-1}_t(g\cdot x))\det\brac{\frac{\partial \psi^{-1}_t}{\partial x}(g\cdot x)}\\
    &=p_0(\psi^{-1}_t(x))\det\brac{\frac{\partial \psi^{-1}_t}{\partial x}(g\cdot x)}\\    &=p_0(\psi^{-1}_t(x))\det\brac{g\cdot\frac{\partial \psi^{-1}_t}{\partial x}(x)\cdot g^{-1}}\\
    &=p_0(\psi^{-1}_t(x))\det\brac{\frac{\partial \psi^{-1}_t}{\partial x}(x)}=p_t(x)\\
\end{align*}
The second equality follows from the $G$-equivariance of $\psi_t^{-1}$ and the $G$-invariance of $p_0$. The third equality is a consequence of the definition of the Jacobian matrix for equivariant functions, and the final equality relies on standard properties of the determinant.
\subsection{Proof of  \texorpdfstring{\cref{lemma:symmetric_flow}}{}}\label{app:symmetric_flow}
\begin{proof}
    The first part of the proof, which involves showing that $\psi_t$ is $G$-symmetric, is straightforward and follows directly from the equivariance properties of $\psi_t$. Since $\psi_t$ is $G\times S_n$-equivariant (and the group actions commute), it is straightforward to see that it is equivariant to each of the groups separately. Let $g \in G$. Then: 
    \begin{equation*}
        g\cdot\psi_t(c)=\psi_t(g\cdot c) = \psi_t(\sigma \cdot c)=\sigma \cdot \psi_t(c)
    \end{equation*}
    Where the first equality follows from the $G$-equivariance of $\psi_t$, the second holds because $c$ is $G$-symmetric,  and the final equality follows from the $S_n$-equivariance of $\psi_t$.
    From the above equation, we also conclude that $c$ and $\psi_t(c)$ are mutually $G$-symmetric. Now, let $g'\in G_{f_i}$, meaning that $g'$ belongs to the site-symmetry group of $f_i$, the $i^{\text{th}}$ fractional coordinate of $c$. Since $g'\in G$ there exists a permutation $\sigma'\in S_n$ s.t $g'\cdot c=\sigma'\cdot c$. Moreover, because $g'\in G_{f_i}$, the permutation must fix the index $i$, $\sigma'(i)=i$. From the previous part of the proof, we know that $g'\cdot\psi_t(c)=\sigma' \cdot \psi_t(c)$ and since $\sigma'(i)=i$,  it follows that $g'\in G_{f'_i}$, where $f'_i$ is the $i^{\text{th}}$ fractional coordinate of $\psi_t(c)$. Therefore,
    $\psi_t(c)$ retains the same site-symmetry structure as $c$, and is thus also $\gW$-constructable.

\end{proof}
\subsection{Proof of  \texorpdfstring{\cref{lemma:W_to_G}}{}}\label{app:W_to_G}
Let $g\in G$. Our goal is to show that there exists a permutation $\sigma \in S_n$ such that $g\cdot\mF_0 = \sigma \cdot\mF_0$. Since $F_0$ is $\gW$-constructable it can be written as a union of orbits under the action of $G$. Consider a single orbit generated by $w_i$, denoted by $\mF_0^{w_i}$. We observe that $g\cdot \mF_0^{w_i}=\sigma' \cdot \mF_0^{w_i}$ for some $\sigma'\in S_{\abs{w_i}}$. This holds because the action of a group element on an orbit is a bijection. Repeating this process for each orbit contained in $\mF_0$ yields a construction for the permutation $\sigma$.
\end{proof}
\subsection{Proof of  \texorpdfstring{\cref{lemma:efficient_GA}}{}}
\begin{proof}
    \begin{equation*}
        \begin{aligned}
        \hat{u}_t(c| G)&= \frac{1}{|G|}\sum_{g\in G}g\cdot u_t(g^{-1}\cdot c) \\
        &= \frac{1}{|G|}\sum_{g\in G}g\cdot u_t(\sigma_{g^{-1}| c}\cdot c) \\
        &= \frac{1}{|G|}\sum_{g\in G}g\cdot \sigma_{g^{-1}| c}\cdot u_t(c)
        \end{aligned}
    \end{equation*}
The second equality follows from the $G$-symmetry of $c$, and the final equality follows from the $S_n$-equivariance of $u_t$.
\end{proof}

\section{Conditional Flow on Fractional Coordinates}\label{app:fractional_coordinates_flow}
\begin{lemma}
    Let $G$ be a space group, the flat tori logarithmic map $\log_{\mF_0}(\mF_1)$ is $G\times S_n$ equivariant.
\end{lemma}
\begin{proof}
    Let $g\in G$ such that $g=(\mR,\tau)$. Since the orthogonal components of $G$ map the crystal to itself, they preserve the lattice structure. Combining this observation with the following lemma (which is expressed with respect to a single point), we obtain that the logarithmic map is equivariant with respect to the space group and that the representation that acts on the output domain includes only the orthogonal part without the translation. The $S_n$ equivariance is trivial for an element-wise function.
\end{proof}
\begin{lemma}
    Let $g\in G$ such that $g=(\mR,\tau)$. If $\mR$ maps $\sZ^3$ to itself, then $\mR\log_{x}(y) = \log_{g\cdot x}(g\cdot y)$.
\end{lemma}
\begin{proof}
    Let $\log_{x}(y)=v$. This means that there exists $z\in\sZ^3$ s.t $v=y-x + z$ where $v\in  [-\frac{1}{2},\frac{1}{2})^3$. Now let's assume $\log_{g\cdot x}(g\cdot y)=v'$. This means that there exists $z'\in\sZ^3$ s.t $v'=g\cdot y-g\cdot x + z'$. Plugging in $g=(\mR,\tau)$ results in $v'=\mR(y-x)+z'$. Combining both equations, we obtain $v'=\mR v+z''$ (because $\mR z\in\sZ^3$). Since $v'\in  [-\frac{1}{2},\frac{1}{2})^3$ and $\norm{v}=\norm{\mR v}$ we conclude that $v'=\mR v$.

\end{proof}
\begin{lemma}
    The conditional flow $\psi_t(\mF_0|\mF_1)$ is $G$-symmetric and $\gW$-constructable.
\end{lemma}
\begin{proof}
    \begin{align*}
       g\cdot\psi_t( \mF_0| \mF_1) &= (\mF_0+ t\log_{\mF_0}(\mF_1))R^T +\one_n\tau^T \\
       &= \mF_0\mR^T+\one_n\tau^T + t\log_{\mF_0}(\mF_1)R^T \\
       &= g\cdot \mF_0 + t\log_{g\cdot\mF_0}(g\cdot\mF_1) \\
       &= \sigma\cdot \mF_0 + t\log_{\sigma\cdot\mF_0}(\sigma\cdot\mF_1) \\
       & = \sigma\cdot\psi_t( \mF_0| \mF_1)
   \end{align*}
The fact that $\psi_t( \mF_0| \mF_1)$ is $\gW$-constructable follows directly from the proof in \cref{app:symmetric_flow} and the fact that $\psi_t( \mF_0| \mF_1)$ is mutually $G$-symmetric with $\mF_0$ and $\mF_1$. 
\end{proof}
\begin{figure}
    \centering
    \includegraphics[width=0.8\linewidth]{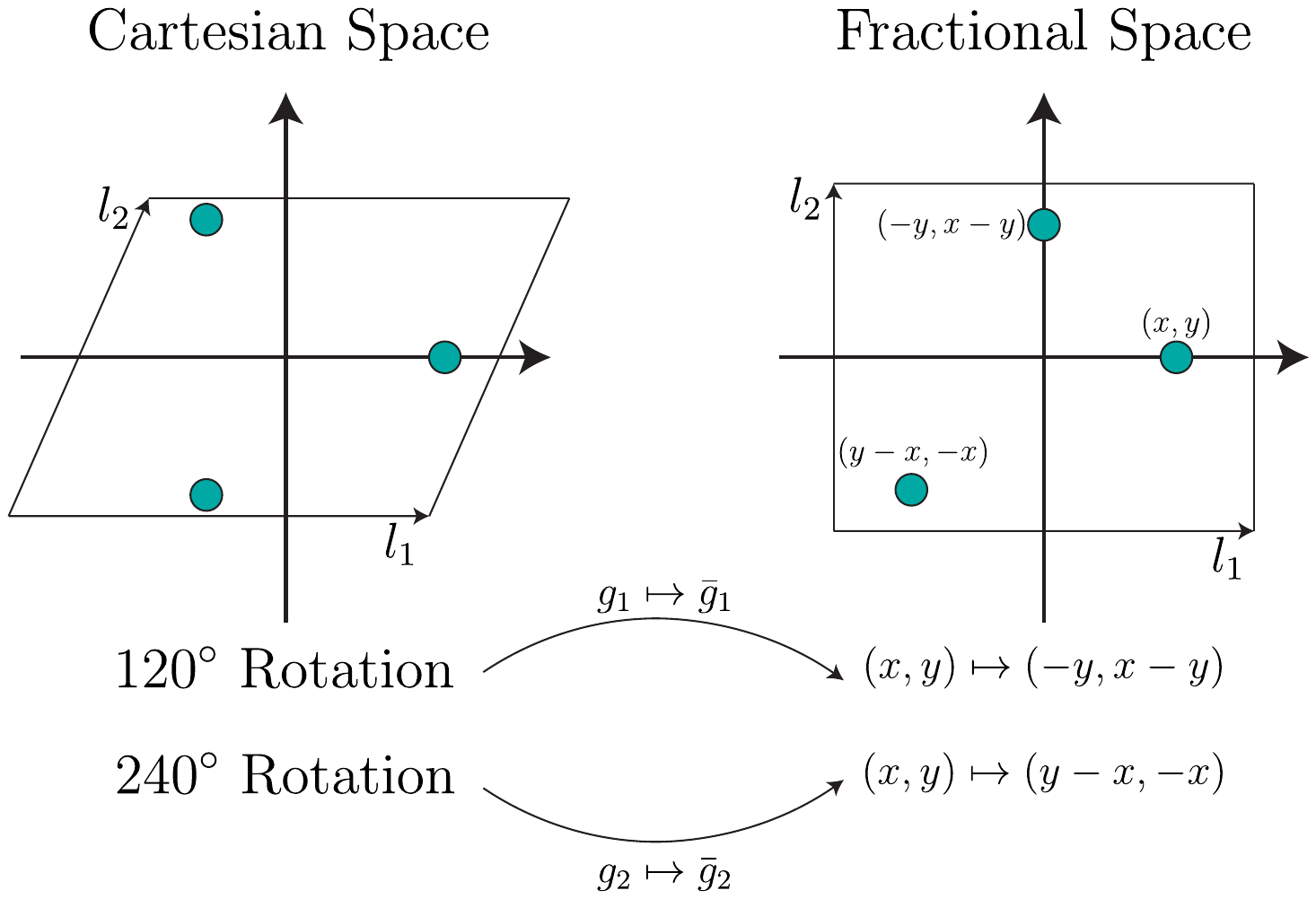}
    \vspace{-10pt}
    \caption{$2$D visualization of the deformed geometry induced by moving to fractional coordinates in a non-orthogonal lattice basis. This example demonstrates how a $3$-fold rotational space group becomes a set of special affine transformations when acting on fractional coordinates and, specifically, the mapping of each rotation element to a corresponding special affine transformation. In addition, the figure illustrates how a space group induced orbit transforms under this change of basis. This new representation makes it difficult to use models that rely on Euclidean geometry.}
    \label{fig:hex_example}
\end{figure}
\section{Lattice Representation}\label{app:lattice}
The lattice matrix $\mL \in \mathbb{R}^{3 \times 3}$ characterizes the geometry of the unit cell. When $\mL$ corresponds to a physically valid lattice, i.e., it has positive volume, it is invertible and can be decomposed into the product $\mL=\mQ\exp(\mS)$ where $\mQ\in\Real^{3\times 3}$ is an orthogonal matrix and $\mS\in \Real^{3\times 3}$ is a symmetric matrix. Representing the lattice parameters via $\mS$ enjoys the benefits of orthogonal invariance (any orthogonal transformation is added to $\mQ$), which makes this representation invariant to any space group operations. 
\citet{jiao2024space} suggested  representing $\mS$ using the coefficients of the following basis - 
\begin{align*}
\mB_1 = \begin{pmatrix}0&1&0\\1&0&0\\0&0&0 \end{pmatrix}, 
\mB_2 = \begin{pmatrix}0&0&1\\0&0&0\\1&0&0 \end{pmatrix},
\mB_3 = \begin{pmatrix}0&0&0\\0&0&1\\0&1&0 \end{pmatrix}, \\
\mB_4 = \begin{pmatrix}1&0&0\\0&-1&0\\0&0&0 \end{pmatrix},
\mB_5 = \begin{pmatrix}1&0&0\\0&1&0\\0&0&-2 \end{pmatrix}, 
\mB_6 = \begin{pmatrix}1&0&0\\0&1&0\\0&0&1 \end{pmatrix}.
\end{align*}
This basis enables clustering of the crystallographic space groups based on the basis coefficients used to represent $\mS$. \Cref{tab:l_cons} summarizes the lattice and coefficient constraints for each crystal family type.
\begin{table}[hbt]
\centering
\caption{Relationship between the lattice shape and the constraint of the symmetric bases.}
\label{tab:l_cons}
\begin{tabular}{cccc}
\toprule
   Crystal Family & Space Group No. & Lattice Shape &  Constraint of Symmetric Bases \\
\midrule\midrule
    Triclinic & $1\sim 2$  & No Constraint & No Constraint \\\midrule
    Monoclinic & $3\sim 15$   & $\alpha=\gamma=90^{\circ}$ & $k_1=k_3=0$ \\\midrule
    Orthorhombic & $16\sim 74$  & $\alpha=\beta=\gamma=90^{\circ}$ & $k_1=k_2=k_3=0$ \\\midrule
    \multirow[c]{2}{*}{Tetragonal} & \multirow[c]{2}{*}{$75\sim 142$}  & $\alpha=\beta=\gamma=90^{\circ}$ & $k_1=k_2=k_3=0$ \\
    & & $a=b$ & $k_4=0$ \\\midrule
    \multirow[c]{2}{*}{Hexagonal} & \multirow[c]{2}{*}{$143\sim 194$}  & $\alpha=\beta=90^{\circ}, \gamma=120^{\circ}$ & $k_2=k_3=0, k_1=-log(3)/4$ \\
    & & $a=b$ & $k_4=0$ \\\midrule
    \multirow[c]{2}{*}{Cubic} & \multirow[c]{2}{*}{$195\sim 230$}  & $\alpha=\beta=\gamma=90^{\circ}$ & $k_1=k_2=k_3=0$ \\
    & & $a=b=c$ & $k_4=k_5=0$ \\
\bottomrule
\end{tabular}
\vspace{-3ex}
\end{table}

\section{Density Functional Theory}
\label{app:dft}
Crystals exist in competition for stability between alternatives with the same composition. If one plots energy against composition, the lowest energy structures form a \emph{convex hull}. We say a crystal is thermodynamically stable if it is near or below this convex hull. Since we do not know all structures, there is epistemic uncertainty in this characterization. Ultimately we report an empirical approximation to the convex hull that may contain no points on the ``true'' convex hull; however, it is the best estimate we know of. The difference between the energy of a crystal and this convex hull is denoted $E_{\text{hull}}$. We report $E_{\text{hull}}<100$ meV/atom and $E_{\text{hull}}<0$ meV/atom rates for stability metrics. These values are computed by prerelaxation with a machine learning interatomic potential \cite{barroso2024open, wood2025umafamilyuniversalmodels} followed by relaxation and energy evaluation using density functional theory.

For the stability metrics, we applied the Vienna ab initio simulation package (VASP) \cite{kresse1996efficient} to compute relaxed geometries and ground state energies at a temperature of 0 K and pressure of 0 atm. 
We used the default settings from the Materials Project \cite{jain2013materials} known as the \texttt{MPRelaxSet} with the PBE functional \cite{perdew1996generalized} and Hubbard U corrections. These correspond with the settings that our prerelaxation network OMat24 \cite{barroso2024open} was trained on, so prerelaxation should reduce DFT energy, up to fitting error.

We did \emph{not} make any guesses about oxidation states! This deviates from the Materials Project which does make those guesses. For this reason, our energy above hull calculations for structures that need to consider oxidation state are slightly high, implying that we might be under-predicting stability. This applies to any stability result we calculated. We expect it to also be a negligible effect.

The results from DiffCSP++ WyFormer in \cref{tab:dng_results} were computed by \citet{kazeev2025wyckofftransformer} and differ slightly from ours. Specifically, they perform multiple relaxations to avoid errors that come from using a poor initial guess before relaxation. Since we prerelax with OMat24 we expect that double relaxation is unnecessary. Consult their work for further details, but we believe the differences are negligible for this purpose.

We contribute a DNG model trained on MP-$20$. We evaluate against a Materials Project reference hull called \texttt{2023-02-07-ppd-mp.pkl.gz} distributed by \texttt{matbench-discovery} \citep{riebesell2023matbench}.

\alex{We contribute models trained on MP-$20$ and Alex-MP-$20$. Models trained on MP-$20$ are evaluated against a materials project hull called \texttt{2023-02-07-ppd-mp.pkl.gz} distributed by \texttt{matbench-discovery} \citep{riebesell2023matbench}. Models trained on Alex-MP-$20$ are evaluated against the Alex-MP-$20$ hull from \citet{zeni2023mattergen}. We note that naive combination of the structures and energies from Alexandria \citep{schmidt2022large, schmidt2022alexandria_2} and Materials Project \citep{jain2013materials} by putting those structures and energies into an \texttt{pymatgen.analysis.phase diagram} \citep{ong2013python} would not result in accurate estimates of thermodynamic stability. The hull distributed by \citet{zeni2023mattergen} recomputed the energy from Alexandria with the Materials Project settings before creating the phase diagram.}

\section{Related Work}
\label{app:related_work_appendix}

As a continuation from \cref{sec:related_work}, we discuss other related work. We still limit the focus to the most relevant parts of this large body of literature.

Our method resembles non-deep learning based methods that propose structures using Wyckoff positions as inductive bias \cite{glass2006uspex, pickard2011ab, meredig2013hybrid, michel2014symmetry} and refine the atomic positions using density functional theory. This field is known as high-throughput screening of inorganic crystals and it is responsible for generating several important datasets of stable materials \cite{saal2013materials, kirklin2015open, wang2021predicting, schmidt2022large, schmidt2022alexandria_2}. Recently, those searches have been sped up by machine learning interatomic potentials that closely approximate density functional theory \cite{merchant2023scaling}. 

Now we take a step further away conceptually to discuss methods that are tangentially related to ours. \emph{Crystal-GFN} \cite{milaai4science2023crystalgfn} is a G-flow network that uses the space group, but does not consider Wyckoff positions. \emph{Mat2Seq} \cite{yan2025mat2seq} proposes a one-dimensional sequence representation of crystal structures that embeds space group information. \emph{GMTNet} \cite{yan2024gmtnet} enforces space group invariance in a crystal property prediction model. Several other works generate crystals without considering multiple types of atom \cite{wirnsberger2022normalizing}, or molecule \cite{kohler2023rigid}. Additionally, there is a large and growing body of literature on generative models for materials that do not have general space group equivariance \cite{xie2021crystal, yang2023scalable, zeni2023mattergen, miller2024flowmm, sriram2024flowllm, lin2024equivariant, joshi2025adit, hoellmer2025omat}.

\section{Evaluation Details}
\label{app:evaluation}
\paragraph{De Novo Generation} In this task, we assess the generative model's capacity to produce thermodynamically stable, unique, and novel crystals (S.U.N.). Here, thermodynamic stability implies a structure is at or near a local minima in composition space, see \Cref{app:dft} for more information; uniqueness is measured across the generated ensemble; and novelty signifies that the structures are absent from both the training and validation sets, computed using \texttt{StructureMatcher} \citep{ong2013python} with default settings (\texttt{ltol=0.2, stol=0.3, angle\_tol=5}).

We additionally reported S.S.U.N. results. This analysis requires detecting the symmetry of a DFT-relaxed structure. We used \texttt{spglib}'s symmetry determination software with parameters \texttt{angle\_tolerance}=\texttt{5.0} and \texttt{symprec}=\texttt{0.1} \citep{spglib}.

Additional proxy metrics are also reported as part of the evaluation.
\emph{Structural validity} implies that the pairwise atomic distances of a crystal's atoms are all greater than $0.5$\AA. \emph{Compositional validity} implies that a crystal has a neutral charge according to the SMACT \cite{davies2019smact} rules. The \emph{properties} that we consider for computing divergences include: $\rho$ the atomic density defined by number of atoms divided by unit cell volume, $\text{elem}$ (arity) the number of unique elements in a crystal, and $\text{cn}$ (coordination number) or the  number of bonds per atom on average. 

\paragraph{Crystal Structure Prediction}
For evaluation, a crystal structure is generated for each entry in the test set and then compared against the corresponding ground truth structure using pymatgen \texttt{StructureMatcher} \citep{ong2013python} with the same threshold values as in \citet{jiao2024space} \texttt{ltol=0.3, stol=0.5, angle\_tol=10}.
We report two metrics: the match rate (MR), defined as the fraction of generated structures that successfully match their ground truth counterparts, and the RMSE, averaged over all matched pairs.

\section{Baselines}\label{app:baseline}
We provide additional context on the core approach behind each baseline we compared against. \textit{CDVAE} integrates a diffusion model with a variational autoencoder for crystal structure generation. \textit{ADiT} uses a latent-based diffusion model and is trained on additional information from the QM$9$ \cite{wu2018qm9} dataset.
\textit{FlowMM} is an application of Riemannian Flow Matching \cite{chen2024flowRiemann} that incorporates non-trivial geometries in the crystal representation space.
\textit{FlowLLM} combines FlowMM with a large language model that serves as a base-distribution generator.
\textit{OMatG} leverages Stochastic Interpolants \cite{albergo2023stochasticinterpolants} for material generation. \textit{SymmCD} operates on the asymmetric unit and incorporates Wyckoff positions as part of the generative process. \textit{DiffCSP++} is a diffusion-based model that conditions on space groups and projects each denoising step through Wyckoff position transformations. \textit{WyFormer} employs an autoregressive model to generate Wyckoff positions conditioned on the space group and subsequently uses DiffCSP++ for full crystal-structure generation. Finally, \textit{SGEquiDiff} is a diffusion-based model that enforces space-group equivariance while operating on the asymmetric unit.

\section{Model Details}\label{app:model_details}
\subsection{Architecture}
In this section, we present a comprehensive overview of our vector field model $\hat{u}_t(\cdot\,|\,G)$, along with the hyperparameters employed during training and generation across all experiments. The model takes as input a crystal $c = (k, \mF, \mA)$, where $f^i \in \Real^3$ denotes the $i^{\text{th}}$ fractional coordinate in $\mF$, and $a^i \in \set{0,1}^h$ represents the $i^{\text{th}}$ atom type indicator vector in $\mA$. The forward computation of the $s$-layer model $\hat{u}(c\mid G)$ is defined by the following set of equations:
\begin{align*}
    a^{i}_{embed} &= \phi^a(a^{i}) \\ 
    t_{embed} &= \text{SinusoidalTimeEmbedding}(t) \\
    h^i_{(0)} &= \phi^\text{embed}(\brac{a^i_{embed},t_{embed}})\\
    m^{ij}_{(l)} &= \phi^\text{edge}_{(l)}(\brac{h^i_{(l-1)}, h^j_{(l-1)}, k,\text{SinusoidalEmbedding}(\log_{f^i}(f^j)),\frac{L^TL\log_{f^i}(f^j)}{\norm{L^TL\log_{f^i}(f^j)}}})\\
    m^i_{(l)} &= \frac{1}{n}\sum_{j=1}^n m^{ij}_{(l)} \\
    h^i_{(l)} &= \phi^\text{node}_{(l)}(\brac{h^i_{(l-1)},m^i_{(l)}})\\
    u^k_t(c_t) &=  \phi^k(\frac{1}{n}\sum_{j=1}^n h^{j}_{(s)}) \\
    (u^\mF_t(c_t))^i &=  \phi^{\mF}(h^i_{(s)}) \\
    (u^\mA_t(c_t))^i &=  \phi^{\mA}(h^i_{(s)}) \\
    u_t(c_t) &= (u^k_t(c_t),u^\mF_t(c_t),u^\mA_t(c_t)) \\
    \hat{u}_t(c_t\,|\, G)&= \frac{1}{|G|}\sum_{g\in G}g\cdot \sigma_{g^{-1}| c}\cdot u_t(c_t)
\end{align*}
We denote $d$ as the hidden dimension of the model, $d_t$ as the Sinusoidal Time Embedding dimension, and $d_s$ as the Sinusoidal Embedding dimension. Next, we list the learnable modules constructing the model and denote their input and output dimension as $x\to y$.   $\phi^a$ is a linear layer $h\to d$, $\phi^\text{embed}$ is a linear layer $d+d_t\to d$, $\phi^\text{edge}_{(l)}$ is $2$-layer Multi-layer Perceptron (MLP) $2d+d_s+9\to d$, $\phi^\text{node}_{(l)}$ is a $2$-layer Multi-layer Perceptron (MLP) $2d \to d$, $\phi^k$ is a linear layer $d\to 6$, $\phi^{\mF}$ is a linear layer $d\to 3$ and $\phi^{\mA}$ is a linear layer $d\to h$. The last equation represents the group averaging presented in \cref{eq:efficient_GA}. The flat tori logarithmic map is defined by the equation:
\begin{equation}\label{eq:logmap}
    \log_{f^{i}}(f^{j})= \frac{1}{2\pi} \operatorname{atan2}(\brac{\sin(f^{j} - f^{i}), \cos(f^{j} - f^{i})})
\end{equation}
\Cref{tab:arch} summarizes the hyperparameters used to train our SGFM models. Note that the same configuration was applied uniformly across all datasets and tasks. The hyperparameter search was performed on MP-$20$ (CSP) and the resulting settings were adopted for all other experiments.
\begin{table}[h]
\centering
\small
\renewcommand{\tabcolsep}{1.6pt}
\caption{Hyperparameter details for all the models reported in the paper. Hyperparameter (bottom row) search was conducted on the MP-$20$ dataset.}
\vspace{10pt}
\begin{tabular}{lcccccc}
\toprule
Dataset      & Number of Layers & $d$               & $d_t$                             & $d_s$                             & Activation            & Layer Norm                \\ \hline
\multicolumn{7}{c}{CSP}                                                                                                                                                         \\ \hline
MP-$20$      & $8$                & $512$             & $256$                             & $128$                             & SiLU                  & \checkmark \\
MPTS-$52$    & $8$                & $512$             & $256$                             & $128$                             & SiLU                  & \checkmark \\
Carbon-$24$  & $8$                & $512$             & $256$                             & $128$                             & SiLU                  & \checkmark \\
Perov-$5$    & $8$                & $512$             & $256$                             & $128$                             & SiLU                  & \checkmark \\
Alex-MP-$20$ & $8$                & $512$             & $256$                             & $128$                             & SiLU                  & \checkmark \\ \hline
\multicolumn{7}{c}{DNG}                                                                                                                                                         \\ \hline
MP-$20$      & $8$                & $512$             & $256$                             & $128$                             & SiLU                  & \checkmark \\ 
\multicolumn{7}{c}{Hyperparameter Range}                                                                                                                                        \\ \hline
MP-$20$      & $\{6,7,8,9\}$    & $\{128,256,512\}$ & \multicolumn{1}{c}{$\{128,256\}$} & \multicolumn{1}{c}{$\{128,256\}$} & \multicolumn{1}{c}{-} & \multicolumn{1}{c}{-}     \\ \bottomrule
\end{tabular}
\label{tab:arch}
\end{table}
\subsection{Training \& Generation}\label{app:train_details}
\begin{table}[htb]
\centering
\small
\renewcommand{\tabcolsep}{1.6pt}
\caption{Training hyperparameter details for all the models reported in the paper. Hyperparameter (bottom row) search was conducted per experiment.}
\vspace{10pt}
\resizebox{\textwidth}{!}{%
\begin{tabular}{lcccccccc}
\toprule
Dataset      & Batch Size/GPU & Learning Rate              & Epochs                                             & $\lambda_{\mF}$                       & $\lambda_{\mA}$                       & $\lambda_k$                           & $s'_{\mF}$ & $s'_{k}$ \\ \hline
\multicolumn{9}{c}{CSP}                                                                                                                                                                                                                 \\ \hline
MP-$20$      & $64$           & $0.0005$                   & $5000$                                             & $100$                                 & -                                     & $1$                                   & $3$ & $3$ \\
MPTS-$52$    & $32$           & $0.0005$                   & $5000$                                             & $100$                                 & -                                     & $1$                                   & $3$ & $3$ \\
Carbon-$24$  & $64$           & $0.0005$                   & $8000$                                             & $100$                                 & -                                     & $1$                                   & $3$ & $3$ \\
Perov-$5$    & $256$          & $0.0005$                   & $1000$                                             & $100$                                 & -                                     & $1$                                   & $3$ & $3$ \\
Alex-MP-$20$ & $32$           & $0.0003$                   & $1250$                                             & $100$                                 & -                                     & $1$                                   & $5$ & $3$ \\ \hline
\multicolumn{9}{c}{DNG}                                                                                                                                                                                                                 \\ \hline
MP-$20$      & $64$           & $0.0007$                   & $5000$                                             & $100$                                 & $1$                                   & $1$                                   & $5$ & $3$ \\
\hline
\multicolumn{9}{c}{Hyperparameter Range}                                                                                                                                                                                                \\ \hline
MP-$20$     & -              & $\{0.0002,0.0005,0.0007\}$ & \multicolumn{1}{l}{$\{1000,1250,2000,5000,8000\}$} & \multicolumn{1}{l}{$\{1,10,50,100\}$} & \multicolumn{1}{l}{$\{1,10,50,100\}$} & \multicolumn{1}{l}{$\{1,10,50,100\}$} & \multicolumn{1}{l}{$\{0, 1, 3 ,5\}$} & \multicolumn{1}{l}{$\{0, 1, 3 ,5\}$} \\ \bottomrule
\end{tabular}}\label{tab:train}
\end{table}
All of our models were trained using the \texttt{ADAM} optimizer \cite{kingma2014adam} on $8$ NVIDIA A$100$ GPUs. \Cref{tab:train} outlines the training configuration for each model, including the ranges explored during hyperparameter search. We employed a cosine annealing learning rate schedule with a minimum learning rate of $0.00001$. As described in \cref{sec:exp}, we applied inference anti-annealing to enhance generation quality.
This technique modifies the vector field by scaling it with a time-dependent function $s(t) = 1 + s't$, where $s' \in \Real^+$ is a hyperparameter. We defined separate annealing parameters for each crystal component: $s'_{\mF}$ and $s'_{k}$ (no annealing was applied to atom type prediction).
All datasets have $60/20/20$ train/validation/test split except Alex-MP-$20$ that has $80/10/10$ split. We used the same split exact indexing split as OMatG to produce a train, test, and validation set \cite{hoellmer2025omat}. OMatG derived its train/test split from MatterGen \cite{zeni2023mattergen} where it took 10\% of the training data to make a validation set.

\subsection{Atom Type Prediction Properties}\label{app:atom_type}
This section explains the properties of the atom-prediction component of $\hat{u}(\cdot \mid G)$. Specifically, it shows that when the input is a symmetric crystal $c$, the predicted atom-type vector field $\hat{u}^{\mathcal{A}}(c \mid G)$ assigns identical values to atoms belonging to the same orbit. A point that may not be immediately obvious is that the group actions are defined in a similar way on the output $\hat{u}(c|G)=(\hat{u}^k(c|G), \hat{u}^\mF(c|G), \hat{u}^\mA(c| G))$ as they are defined over $c$, up to a representation, as discussed in the section on fractional coordinates. 

To begin with, we can show that $\hat{u}^{\mathcal{A}}(c \mid G)$ is $G$–invariant. On one hand, we know that $\hat{u}^A(c\mid G)=\hat{u}^A(g \cdot c\mid G)$ because $\hat{u}$ is $G$–equivariant and the group element $g$ does not act directly on the atom-type component. Moreover, we also have $\hat{u}^A(g \cdot c\mid G) = \hat{u}^A(\sigma_{g|c} \cdot c\mid G)=  \sigma_{g|c} \cdot\hat{u}^A( c\mid G)$  since $\hat{u}$ is also $S_n$–equivariant and $c$ is $G$–symmetric.

This relation, $\hat{u}^A(c\mid G)=\sigma_{g|c} \cdot\hat{u}^A( c\mid G)$, applies only to permutations corresponding to elements of the space group, which only permute atoms within the same orbit induced by the group action. Consequently, we can conclude that the atom-type vector field $\hat{u}^{\mathcal{A}}_t$ takes identical values for atoms belonging to the same orbit.

\section{Ablations}
\label{app:ablations}

\subsection{Running Time Ablations}\label{app:runtime}

\begin{table}[H]
\centering
\caption{Generation times for various batch sizes.}\label{tab:dng_time}

\begin{tabular}{lccc}
\toprule 
Model      & \multicolumn{3}{c}{Generation Time (s/batch)}\\
                       & $64$                        & $256$                        & $500$                       \\ 
\midrule
                
DiffCSP++              & $66$                      & $268$                      & $490$                     \\
SGEquiDiff             & $19$                      & $75$                       & $154$                     \\
SGFM                   & $22$                      & $92$                       & $175$       \\             
\bottomrule
\end{tabular}
\end{table}

\begin{table}[t]
\centering
\caption{Training and generation time comparison of different vector field models.}\label{tab:time_copmarison}
\renewcommand{\tabcolsep}{1.6pt}
\begin{adjustbox}{width=0.35\textwidth}
\begin{tabular}{l|ccc}
\toprule
\multirow{2}{*}{Model} & 
Batch size &
Training
& Generation \\
                       &  & Time (s) & Time (s)   \\ \hline
SGFM                   & $64$               & $28.2$   & $17.81$    \\
Non-Equivariant   & $64$               & $26.3$   & $16.39$    \\
GA              & $1$                & $600$    & $14.2$         \\
\bottomrule

\end{tabular}
\end{adjustbox}
\end{table}
This experimental ablation study aims to evaluate the efficiency of SGFM by comparing its performance during both training and generation against two baseline models: (1) a non-equivariant variant where the vector field does not incorporate group symmetry, and (2) a standard GA implementation as defined in \cref{eq:GA}. For each model, we measured the time required to train a single epoch on MP-$20$, as well as the time needed to generate a batch of 64  (with 500 generation steps). The training time was averaged over 10 epochs, while the generation time was averaged over 100 batches. Training was conducted on an NVIDIA RTX$8000$  using 8 GPUs, while generation was performed on a single NVIDIA A$10$ GPU. The results are summarized in \cref{tab:time_copmarison}. Due to memory constraints, the standard GA model was limited to a maximal batch size of $1$ per GPU. The generation results for the standard GA model were obtained from a randomly initialized model, as training such a model was intractable. The results highlight a significant efficiency gap between the SGFM implementation of GA and the standard version, while showing minimal difference compared to the non-equivariant model.

\subsection{Additional Results}
\begin{table}[]
\centering
\small
\caption{\textbf{CSPML evaluation.} Wyckoff positions are obtained via the CSPML heuristic. The reported metrics are MR and RMSE.}\label{tab:cspml}
\begin{tabular}{l|cc|cc|cc}
\multirow{2}{*}{Model} & \multicolumn{2}{c|}{MP-$20$}           & \multicolumn{2}{c|}{MPTS-$52$}         & \multicolumn{2}{c}{Perov-$5$}          \\
                       & MR (\%) $\uparrow$ & RMSE $\downarrow$ & MR (\%) $\uparrow$ & RMSE $\downarrow$ & MR (\%) $\uparrow$ & RMSE $\downarrow$ \\ \hline
CSPML                  & $70.51$            & $.0388$           & $\textbf{36.98}$            & $.0664$           & $51.84$            & $.1066$           \\
DiffCSP++    & $\textbf{70.58}$            & $.0272$           & $37.17$            & $.0676$           & $52.17$            & $.0841$           \\
SGFM         & $70.13$            & $.0310$           & $35.09$            & $.0739$           & $\textbf{54.10}$            & $.1508$          
\end{tabular}
\end{table}
\begin{table}[]
\centering
\caption{\textbf{Noise prior vs. model equivariance comparison.} The reported result is MR on the MP-$20$ dataset.} \label{tab:model_type_ablation}
\begin{tabular}{l|cc}
Noise Prior \textbackslash Model & Equivariant & Non-Equivariant \\ \hline
Wyckoff Conditional              & $\textbf{82.74}$     & $68.16$           \\
Uniform                          & -           & $64.49$    \\
\end{tabular}
\end{table}
\begin{figure}[H]
\centering
\vspace{-1em}
\includegraphics[width=0.5\linewidth]{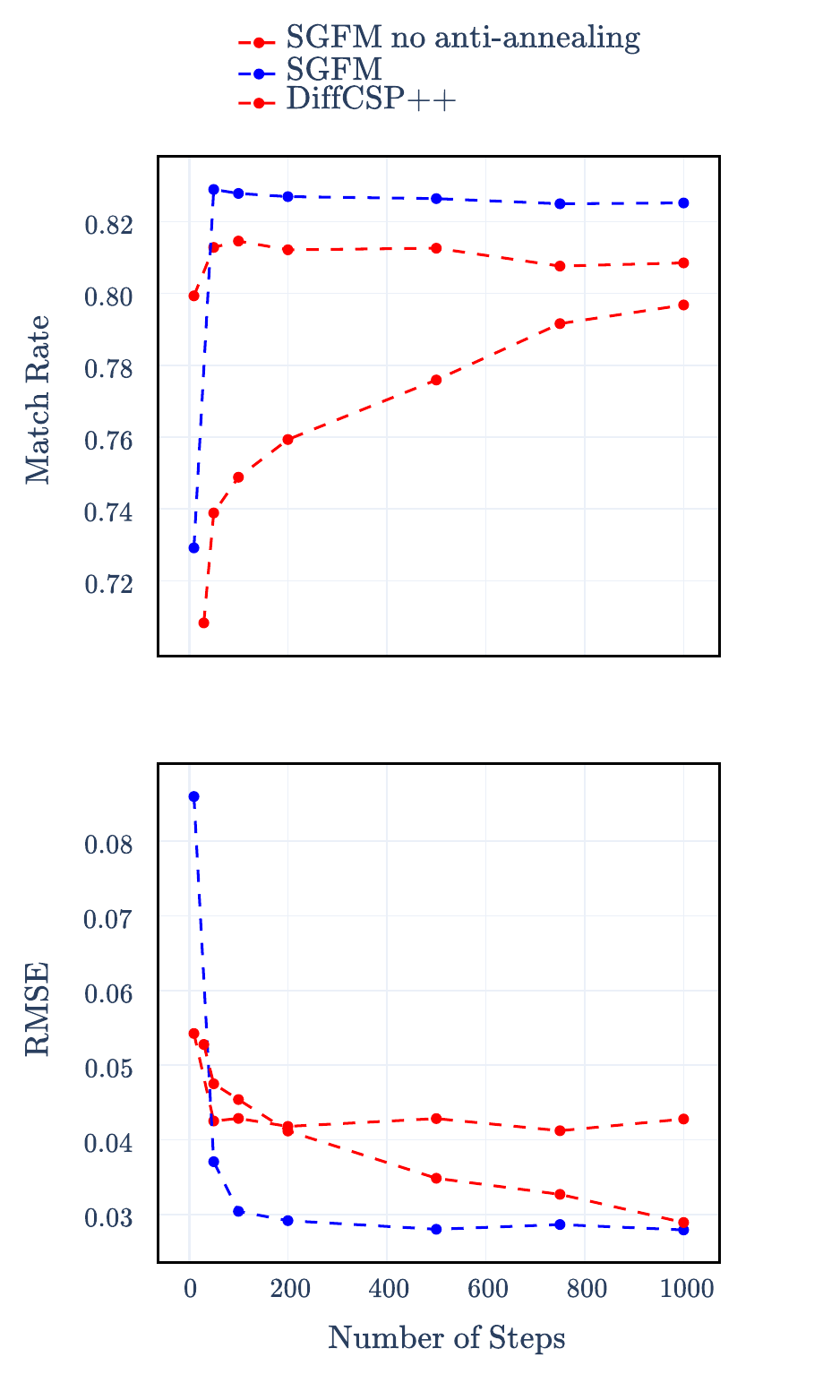}
\vspace{-10pt}
\caption{\textbf{Sampling Efficiency Analysis.} MR ($\uparrow$) and RMSE ($\downarrow$) versus integration steps on MP-$20$.  We observe that SGFM reaches near-optimal accuracy within just 50–100 steps. DiffCSP++ converges more slowly, requiring up to 1000 steps to approach its best performance—while still showing a notable gap in match rate compared to SGFM. In this same plot, we include a version of SGFM without anti-annealing to assess the effects of anti-annealing on performance.}
\label{fig:SGFM_vs_diffcsp++}
\end{figure}

\end{document}